\documentclass[12pt]{colt2021}

\usepackage{times}

\usepackage{lmodern}
\usepackage{hyperref}       
\usepackage{booktabs}       
\usepackage{amsfonts}       
\usepackage{nicefrac}       
\usepackage{microtype}      

\PassOptionsToPackage{numbers, sort, compress}{natbib}

\usepackage{mathtools, verbatim}
\usepackage{color}
\definecolor{darkblue}{rgb}{0.0,0.0,0.2}
\hypersetup{colorlinks,breaklinks,
	linkcolor=darkblue,urlcolor=darkblue,
	anchorcolor=darkblue,citecolor=darkblue}
\usepackage{wrapfig}
\usepackage[small]{caption}
\usepackage[colorinlistoftodos,textsize=tiny]{todonotes} 
\usepackage{bbm}
\usepackage{xspace}
\usetikzlibrary{arrows}
\usetikzlibrary{shapes}

\makeatletter
\let\Ginclude@graphics\@org@Ginclude@graphics 
\makeatother

\newcommand{\Comments}{0}
\newcommand{\mynote}[2]{\ifnum\Comments=1\textcolor{#1}{#2}\fi}
\newcommand{\mytodo}[2]{\ifnum\Comments=1
	\todo[linecolor=#1!80!black,backgroundcolor=#1,bordercolor=#1!80!black]{#2}\fi}

\newcommand{\btw}[1]{}

\ifnum\Comments=1               
\fi

\newcommand{\reals}{\mathbb{R}}

\newcommand{\simplex}{\Delta_\Y}
\newcommand{\relint}[1]{\mathrm{relint}(#1)}
\newcommand{\interior}{\mathrm{int}\,}
\newcommand{\prop}[2][\mathcal{P}]{\mathrm{prop}_{#1}[#2]}
\newcommand{\elic}{\mathrm{elic}}
\newcommand{\eliccvx}{\mathrm{elic}_\mathrm{cvx}}
\newcommand{\conscvx}{\mathrm{cons}_\mathrm{cvx}}

\newcommand{\rank}{\mathrm{rank}}
\newcommand{\proj}{\mathrm{proj}}
\newcommand{\supp}{\mathrm{supp}}
\newcommand{\spn}{\mathrm{span}}

\newcommand{\range}{\mathrm{range}\,}
\newcommand{\zeros}[1]{\mathrm{ker}_\P\,#1}
\newcommand{\codim}{\mathrm{codim}}

\newcommand{\simplexp}{\Delta_{\Y'}}

\newcommand{\propdis}{\mu}
\newcommand{\affhull}{\mathrm{affhull}}

\newcommand{\C}{\mathcal{C}}
\newcommand{\D}{\mathcal{D}}
\newcommand{\E}{\mathbb{E}}

\renewcommand{\L}{\mathcal{L}}
\newcommand{\Lcvx}{\mathcal{L}^{\mathrm{cvx}}}

\newcommand{\R}{\mathcal{R}}
\renewcommand{\P}{\mathcal{P}}
\newcommand{\Sc}{\mathcal{S}}  
\newcommand{\Scr}{\mathcal{S}}  

\newcommand{\V}{\mathcal{V}}
\newcommand{\X}{\mathcal{X}}
\newcommand{\Y}{\mathcal{Y}}

\newcommand{\lbar}{\underline{L}} 

\newcommand{\im}{\mathrm{im}}
\newcommand{\Var}{\mathrm{Var}}
\newcommand{\CVaR}{\mathrm{CVaR}}

\newcommand{\exploss}[3]{\E_{#3} #1(#2,Y)}
\newcommand{\risk}[1]{\underline{#1}}
\newcommand{\Ind}[1]{\mathbf{I}\{{#1}\}}
\newcommand{\inprod}[2]{\langle #1, #2 \rangle}
\newcommand{\toto}{\rightrightarrows}
\newcommand{\ones}{\mathbbm{1}}

\newtheorem{condition}{Condition}

\DeclareMathOperator*{\argmin}{arg\,min}

\usepackage{thmtools, thm-restate}

\title{Unifying Lower Bounds on Prediction Dimension of Consistent Convex Surrogates}

\coltauthor{\Name{Jessie Finocchiaro} \Email{jessica.finocchiaro@colorado.edu}\\
	\Name{Rafael Frongillo} \Email{raf@colorado.edu}\\
	\Name{Bo Waggoner} \Email{bwag@colorado.edu}\\
	\addr CU Boulder}

\begin{document}

\maketitle

\begin{abstract}
Given a prediction task, understanding when one can and cannot design a consistent convex surrogate loss, particularly a low-dimensional one, is an important and active area of machine learning research. 
The prediction task may be given as a target loss, as in classification and structured prediction, or simply as a (conditional) statistic of the data, as in risk measure estimation.
These two scenarios typically involve different techniques for designing and analyzing surrogate losses.
We unify these settings using tools from property elicitation, and give a general lower bound on prediction dimension.
Our lower bound tightens existing results in the case of discrete predictions, showing that previous calibration-based bounds can largely be recovered via property elicitation.
For continuous estimation, our lower bound resolves on open problem on estimating measures of risk and uncertainty.
\end{abstract}

\section{Introduction}\label{sec:intro}
A surrogate loss function is an error measure that is related but not identical to one's target problem of interest.
Selecting a hypothesis by minimizing surrogate risk is one of the most widespread techniques in supervised machine learning.
There are two main reasons why a surrogate loss is necessary: (1) the target loss does not satisfy some desiderata, such as convexity, or (2) the goal is to estimate some target statistic and there is no target loss, as in many continuous estimation problems.
In both settings, a key criteria for choosing a surrogate loss is \emph{consistency}, a precursor to excess risk bounds and convergence rates.
Roughly speaking, consistency means that minimizing surrogate risk corresponds to solving the target problem of interest, i.e. in (1) the target risk is also minimized, or in (2) the continuous prediction approaches the true conditional statistic.

Despite the ubiquity of surrogate losses, we lack general frameworks to design and analyze consistent surrogates.
This state of affairs is especially dire when one seeks low \emph{prediction dimension}, the dimension of the surrogate prediction domain.
For example, in multiclass classification with $n$ labels, the prediction domain might be $\reals^n$.
In many type (1) settings, such as structured prediction and extreme classification, the prediction dimension can easily become intractably large, forcing one to sacrifice consistency for computational efficiency.
To understand whether this sacrifice is necessary, recent work developed tools like the feasible subspace dimension to lower bound the prediction dimension of any consistent convex surrogate~\citep{ramaswamy2016convex}.
Challenges of type (2) include risk measures such as conditional value at risk (CVaR), with applications in financial regulation, robust engineering design, and algorithmic fairness.
Risk measures provably cannot be specified via a target loss, and thus we seek a surrogate loss of low (or at least finite) prediction dimension.
Recent work~\citep{fissler2016higher,frongillo2020elicitation} gives prediction dimension bounds for some of these risk measures, but without the requirement that the surrogate be convex: bounds for convex surrogates are left as a major open question.

We present a unification of existing techniques to bound the prediction dimension of consistent convex surrogates in both settings above.
Applied to settings of type (1), we recover the feasible subspace dimension result of \citet{ramaswamy2016convex}, and give an example where our bound is even tighter.
For type (2), we give the first prediction dimension bounds for risk measures with respect to convex surrogates, addressing the open question above.
Our framework rests on \emph{property elicitation}, a weaker condition than calibration, as a tool to understand consistency across a wide variety of domains.

\paragraph{The ``four quadrants'' of problem types}
Above, we discuss a significant divergence in previous frameworks: constructing a surrogate given a \emph{target loss} versus a \emph{target statistic}.
In addition to the two possible targets, we may have one of two domains: a \emph{discrete} (i.e.\ finite) target prediction space, like a classification problem, or a \emph{continuous} one, like a regression or estimation problem.
We informally refer to the four resulting cases---target loss vs.\ target statistic, and discrete vs.\ continuous predictions---as the ``four quadrants'' of supervised learning problems, shown in Table~\ref{tab:quadrants}.
For further examples, see Appendix~\ref{app:omitted-examples}.

\paragraph{Literature on consistency and calibration}
We focus on the construction of consistent surrogate losses $L: \reals^d \times \Y \to \reals$, roughly meaning that minimizing $L$-loss corresponds to solving the target problem of interest.
When given a \emph{target loss} $\ell$, we roughly define $L$ to be consistent if minimizing $L$, and applying a link function, minimizes $\ell$ (Definition \ref{def:consistent-ell})~\citep{zhang2004statistical,bartlett2006convexity,tewari2007consistency,steinwart2007compare,ramaswamy2016convex}.
When given a target statistic such as the conditional quantile or variance, but no target loss, we introduce a notion of consistency in line with classical statistics (Definition \ref{def:consistent-prop}) \citep{gyorfi2006distribution, fan1998efficient,ruppert1997local}.
Here we will define $L$ to be consistent if minimizing $L$ and applying a link function yields estimates converging to the correct value.

A priori, it is not clear that compatible definitions of consistency could be given for both target statistics and target losses.
In fact, we observe that consistency for target losses is a special case of consistency for target statistics (\S~\ref{sec:consis-implies-indir}).
This observation suggests property elicitation (see \S~\ref{subsec:properties}) as a useful tool to study general lower bounds.

As definitions of consistency are relatively intractable to apply directly, the literature often focuses on a weaker condition called calibration, which only applies when given a discrete target loss, e.g. Quadrants 1 and 3. 
Particularly,~\citet{zhang2004statistical, lin2004note,bartlett2006convexity,tewari2007consistency,ramaswamy2016convex} show the equivalence of consistency and calibration in Quadrant 1, where one is given a target statistic and discrete prediction set.
We discuss the additional relationship of elicitation and calibration in Appendix~\ref{app:calibration}, and derive Theorem~\ref{thm:consistent-implies-indir-elic} via calibration.

\begin{table}
	\begin{center}
	\begin{tabular}{p{12ex}|l|l|} 
		& \emph{Target loss}  & \emph{Target statistic}\\ \hline
		\emph{Discrete} & \textbf{Q1:} Classification  & \textbf{Q2:} Risk-averse classification\\ 
		{\em prediction} & & (Appendix~\ref{app:omitted-examples})\\ \hline 
		\emph{Continuous } & \textbf{Q3:} Least-squares regression & \textbf{Q4:} Variance estimation\\ 
		\emph{estimation}& & \\
		\hline
	\end{tabular}
\caption{The four quadrants of problem types, with an example of interest for each.}
\label{tab:quadrants}
\end{center}
\end{table}

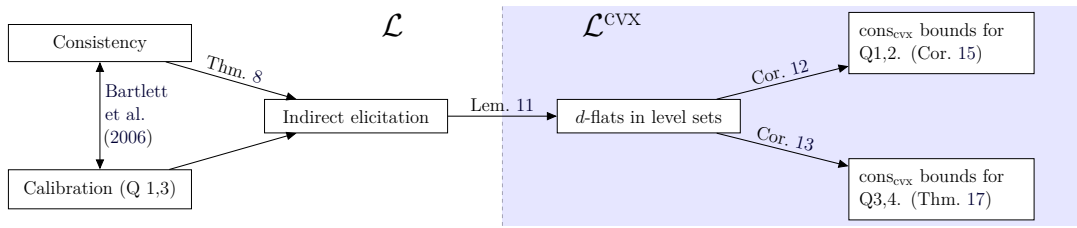
\begin{figure}
\begin{center}
\resizebox{0.95\linewidth}{!}{
\begin{tikzpicture}

\tikzset{vertex/.style = {draw,minimum width=13em, inner sep=8pt, fill = white, font=\Large}}
\tikzset{edge/.style = {->,> = triangle 45, thick}}
\tikzset{node/.style = {anchor=above, sloped}}
\tikzstyle{every node}=[font=\Large]

\path[fill=blue, fill opacity = 0.1] (0,-3) -- (0,3) -- (16,3) -- (16,-3) -- cycle;
\draw[dashed, opacity = 0.3] (0,-3) -- (0,-0.1);
\draw[dashed, opacity = 0.3] (0,3) -- (0,0.6);
\node[draw=none] at (3,2.5) {{\Huge$\Lcvx$}};
\node[draw=none] at (-3,2.5) {{\Huge$\L$}};

\node[vertex] (consistent) at  (-11,2) {Consistency};
\node[vertex] (calibrated) at  (-11,-2) {Calibration (Q 1,3)};
\node[vertex] (indir-elic) at  (-4,0) {Indirect elicitation};

\node[vertex] (d-flats) at  (4,0) {$d$-flats in level sets};
\node[vertex, text width=45mm] (q1-bounds) at (12,2) {$\conscvx$ bounds for Q1,2. (Cor.~\ref{cor:fsd-bound})};
\node[vertex, text width=45mm] (q4-bounds) at (12,-2) {$\conscvx$ bounds for Q3,4. (Thm.~\ref{thm:bayes-risk-lower-bound})};

\draw[edge, <->] (consistent) to node[right,text width=5mm]{\cite{bartlett2006convexity}} (calibrated);
\draw[edge] (consistent) to node[above, sloped]{Thm.~\ref{thm:consistent-implies-indir-elic}} (indir-elic);
\draw[edge] (calibrated) to (indir-elic);
\draw[edge] (indir-elic) to node[above, sloped]{Lem.~\ref{lem:convex-flats-inf-dim}} (d-flats);

\draw[edge] (d-flats) to node[above, sloped]{Cor.~\ref{cor:Pcodim-flat-elic-relint-prop}} (q4-bounds);
\draw[edge] (d-flats) to node[above, sloped]{Cor.~\ref{cor:Pcodim-flat-single-val-prop}} (q1-bounds);
\end{tikzpicture}
}
\end{center}
\caption{Flow and implications of our results. Compared to calibration, we suggest indirect elicitation as a simpler but almost-as-powerful necessary condition for consistency. In particular, we obtain a testable necessary condition, based on $d$-flats, for whether there exists a $d$-dimensional consistent convex surrogate. This condition recovers and strengthens existing calibration-based results.}
\label{fig:results-flow}
\end{figure}

\paragraph{Contributions}
First, we formalize a notion of consistency with respect to a target statistic (Definition~\ref{def:consistent-prop}) and show its relationship to consistency with respect to a target loss (Lemma~\ref{lem:consistent-loss-implies-prop}).
We then show indirect elicitation is a necessary condition for consistency (Theorem~\ref{thm:consistent-implies-indir-elic}).
With these tools in hand, we present a new framework for deriving lower bounds on the prediction dimension of consistent convex surrogates (Corollaries~\ref{cor:Pcodim-flat-single-val-prop} and~\ref{cor:Pcodim-flat-elic-relint-prop}) via indirect elicitation.
These bounds are the first to our knowledge that can be applied in all four quadrants.
Moreover, our framework can also give tighter bounds than previously existed in the literature.
We illustrate this sharpness with new bounds for well-studied problems such as abstain loss (\S~\ref{sec:finite-calib}) and variance, CVaR, and other measures of risk and uncertainty (\S~\ref{sec:contin-consis}).
See Figure~\ref{fig:results-flow} for a roadmap of our main results.

\section{Setting}\label{sec:related-work}

We consider supervised learning problems in the space $\X \times \Y$, for some \emph{feature space} $\X$ and a \emph{label space} $\Y$, with data drawn from a distribution $D$ over $\X \times \Y$.
The task is to produce a hypothesis $f: \X \to \R$, for some \emph{prediction space} $\R$, which may be different from $\Y$.
For example, in ranking problems, $\R$ may be all $|\Y|!$ permutations over the $|\Y|$ labels forming $\Y$.
As we focus on conditional distributions $p := D_x = \Pr[Y|X=x]$ over $\Y$ given some $x \in X$, we often abstract away $x$, working directly with a convex set of distributions over outcomes $\P \subseteq \simplex$.
We then write e.g.\ $\E_p L(\cdot,Y)$ to mean the expectation when $Y\sim p$.

If given, we use $\ell: \R \times \Y \to \reals$ to denote a \emph{target loss}, with predictions $r\in\R$.
Similarly, $L: \reals^d \times \Y \to \reals$ will typically denote a \emph{surrogate} loss, with surrogate predictions $u \in \reals^d$.
We write $\L_d $ for the set of $\mathcal{B}(\reals^d) \otimes \Y$-measurable and lower semi-continuous surrogates $L : \reals^d \times\Y \to \reals$ such that $\E_{Y \sim p} L(u,Y) < \infty$ for all $u \in \reals^d, p \in \P$, that are minimizable in that $\argmin_{u} \exploss{L}{u}{p}$ is nonempty for all $p\in\P$.
Moreover, $\Lcvx_d \subseteq \L_d$ is the set of convex (in $\reals^d$ for every $y \in \Y$) losses in $\L_d$.
Set $\L = \cup_{d \in \mathbb{N}} \L_d$, and $\Lcvx = \cup_{d \in \mathbb{N}} \Lcvx_d$.
A loss $\ell: \R\times \Y \to \reals$ is \emph{discrete} if $\R$ is a finite set.
For a given $p\in\P$, the (conditional) \emph{regret}, or excess risk, of a loss $L$ is given by $R_L(u,p) := \exploss{L}{u}{p} - \inf_{u^*} \exploss{L}{u^*}{p}$.
Typically, we notate finite report sets $\R$.

\subsection{Property elicitation}\label{subsec:properties}
Arising from the statistics and economics literature, property elicitation is similar to calibration, but only characterizes exact minimizers of a surrogate~\citep{savage1971elicitation,osband1985information-eliciting,lambert2008eliciting,lambert2009eliciting,lambert2018elicitation,frongillo2015vector-valued,frongillo2014general}.
Specifically, given a statistic or \emph{property} $\Gamma$ of interest, which maps a distribution $p \in \P \subseteq \simplex$ to the set of desired or correct predictions, the minimizers of $L$ should precisely coincide with $\Gamma$.
For example, squared loss $L(r,y) = (r-y)^2$ elicits the mean $\Gamma(p) = \E_p Y$.
For intuition, to relate to consistency, one can think of $p = \Pr[Y|X=x]$ as a conditional distribution, though the definition is also applied to point prediction settings.

\begin{definition}[Property, elicits]
	A \emph{property} is a set-valued function $\Gamma : \P \to 2^\R \setminus \{\emptyset\}$, which we denote $\Gamma: \P \toto \R$.
	A loss $L : \R \times \Y \to \reals$ \emph{elicits} the property $\Gamma$ if
	\begin{equation}
    \label{eq:elic}    
    \forall p \in \P, \;\; \Gamma(p) = \argmin_{u \in \R} \exploss{L}{u}{p}~.
	\end{equation}
\end{definition}
An example is the mean, $\Gamma(p) = \{\E_p Y\}$.
The \emph{level set} of $\Gamma$ at value $r\in\R$ is $\Gamma_r := \{p \in \P : r \in \Gamma(p)\}$.
We call a property $\Gamma: \P \toto \R$ \emph{discrete} if $\R$ is a finite set, as in Quadrants 1 and 2.
A property is \emph{single-valued} if $|\Gamma(p)|=1$ for all $p\in\P$, in which case we may write $\Gamma:\P\to\R$ and $\Gamma(p) \in \R$.
The mean is single-valued.
We define the \emph{range} of a property by $\range \Gamma = \bigcup_{p\in\P} \Gamma(p) \subseteq \R$.
When $L\in\L$, we use $\Gamma := \prop[\P]{L}$ to denote the unique property elicited by $L$ (for distributions in $\P$) from eq.~\eqref{eq:elic}. 
Typically, we denote the target property by $\gamma$, and the surrogate by $\Gamma$.

To relate property elicitation to consistency, we need to allow for a link function, which gives rise to the notion of \emph{indirect} elicitation.
For single-valued properties, this definition reduces to the natural requirement $\gamma = \psi \circ \Gamma$.
\begin{definition}[Indirect Elicitation]\label{def:indirectly-elicits}
	A surrogate loss and link $(L, \psi)$ \emph{indirectly elicit} a property $\gamma:\P \toto \R$ if $L$ elicits a property $\Gamma: \P \toto \reals^d$ such that for all $u \in \reals^d$, we have $\Gamma_u \subseteq \gamma_{\psi(u)}$.
	We say $L$ \emph{indirectly elicits} $\gamma$ if such a link $\psi$ exists.
\end{definition}

An important caveat to the above definitions is that, since $\Gamma = \prop[\P]{L}$ is nonempty everywhere, we must have $L\in\L$, meaning that $\exploss{L}{\cdot}{p}$ always achieves a minimum.
This restriction is also implicit in e.g.~\citep{agarwal2015consistent}.
While some popular surrogates such as logistic and exponential loss are not minimizable, these losses are still covered in Corollary~\ref{cor:Pcodim-flat-elic-relint-prop} and Theorem~\ref{thm:bayes-risk-lower-bound} as $\Gamma(p) \neq \emptyset$ when $p\in\P := \relint\simplex$; moreover, by thresholding $L'(u,y) = \max(L(u,y),\epsilon)$ for sufficiently small $\epsilon>0$ we can achieve $L'\in\L$ for both.
We expect that a generalization of property elicitation which allows for ``infinite'' predictions (e.g., along a prescribed ray), thereby ensuring a minimum is always achieved for convex losses, would allow us to lift the minimizable restriction entirely.

\subsection{Convex consistency dimension and elicitation complexity}\label{subsec:complexity}

Various works have studied the minimum prediction dimension $d$ needed in order to construct a consistent surrogate loss $L: \reals^d \times \Y \to \reals$, typically through proxies such as calibration~\citep{steinwart2008support,agarwal2015consistent,ramaswamy2016convex} and property elicitation~\citep{frongillo2015vector-valued,fissler2016higher,frongillo2020elicitation}.
In Quadrant 1, \citet{ramaswamy2016convex} introduce a special case of convex consistency dimension (Definition~\ref{def:cvx-consistency-dim}), which led to consistent convex surrogates for discrete prediction problems such as hierarchical classification~\citep{ramaswamy2015hierarchical} and classification with an abstain option~\citep{ramaswamy2018consistent}.

\begin{definition}[Convex Consistency Dimension]\label{def:cvx-consistency-dim}
  Given target loss $\ell:\R \times\Y \to \reals$ or property $\gamma: \P \toto \R$, its \emph{convex consistency dimension} $\conscvx(\cdot)$ is the minimum dimension $d$ such that $\exists L \in \Lcvx_d$ and link $\psi$ such that $(L,\psi)$ is consistent with respect to $\ell$ or $\gamma$.
\end{definition}
\emph{Consistency} is defined for a target loss in Definition~\ref{def:consistent-ell} and for a target property in Definition~\ref{def:consistent-prop}.

In the case of a target property $\gamma$, i.e. a statistic, \citet{lambert2008eliciting} similarly introduce the notion of \emph{elicitation complexity}, later generalized by \citet{frongillo2020elicitation}, which captures the lowest prediction dimension of a surrogate which indirectly elicits $\gamma$.
This notion is quite general as it includes continuous estimation settings and does not inherently depend on a target loss being given. 

\begin{definition}[Convex Elicitation Complexity]\label{def:cvx-elic-complex}
	Given a target property $\gamma$, the \emph{convex elicitation complexity} $\eliccvx(\gamma)$ is the minimum dimension $d$ such that there is a $L \in \Lcvx_d$ indirectly eliciting $\gamma$.
\end{definition}

\citet[Corollary 10]{agarwal2015consistent} provide a necessary condition for the direct convex elicitation of single-valued properties, yielding bounds on the dimensionality of level sets.
Moreover, \citet{finocchiaro2019embedding} study surrogate losses which \emph{embed} a discrete loss, which is a special case of indirect elicitation.
\citet{finocchiaro2020embedding} further introduce the notion of \emph{embedding dimension}, which is a lower bound on both convex elicitation complexity of discrete properties and convex consistency dimension of discrete losses and finite statistics.

\section{Consistency implies indirect elicitation}\label{sec:consis-implies-indir}

In this section, we connect consistency of any surrogate to an indirect elicitation requirement.
This will allow us to show indirect elicitation gives state-of-the-art lower bounds on the prediction dimension of consistent convex surrogates.

We start by formalizing consistency in two ways that generalize across our four quadrants.
First, given a target loss $\ell$, we say $L$ is consistent if optimizing $L$ and applying a link $\psi$ optimizes $\ell$ (Definition~\ref{def:consistent-ell}).
Second, given a target property $\gamma$, such as the $\alpha$-quantile, we say $L$ is consistent if optimizing $L$ implies approaching, in some sense, the correct statistic $\gamma(D_x)$ of the conditional distributions $D_x = \Pr[Y|X=x]$ (Definition~\ref{def:consistent-prop}).
We then observe that Definition~\ref{def:consistent-ell} is subsumed by Definition~\ref{def:consistent-prop}, and use this to show consistency implies $L$ indirectly elicits $\prop{\ell}$ or $\gamma$ respectively.

\begin{definition}[Consistent: loss]\label{def:consistent-ell}
  A loss $L \in \L$ and link $(L,\psi)$ are \emph{$\D$-consistent} for a set $\D$ of distributions over $\X \times \Y$ with respect to a target loss $\ell$ if, for all $D \in \D$ and all sequences of measurable hypothesis functions $\{f_m : \X \to \R\}$,
  \begin{align*}
    \E_D L(f_m(X), Y) \to \inf_f \E_D L(f(X), Y) &\implies \E_D \ell((\psi \circ f_m)(X), Y) \to \inf_f \E_D \ell((\psi \circ f)(X), Y)~.~
  \end{align*}
  For a given convex set $\P \subseteq \simplex$, we simply say $(L,\psi)$ is \emph{consistent} if it is $\D$-consistent for some $\D$ satisfying the following: for all $p \in \P$, there exists $D \in \D$ and $x \in \X$ such that $D$ has a point mass on $x$ and $p = D_x$.
\end{definition}

Instead of a target loss $\ell$, one may want to learn a target property, i.e. a conditional statistic such as the expected value, variance, or entropy.
In this case, following the tradition in the statistics literature on conditional estimation~\citep{gyorfi2006distribution,fan1998efficient,ruppert1997local},
we formalize consistency as converging to the correct conditional estimates of the property.
Convergence is measured by functions $\propdis(r, p)$ that formalize how close $r$ is to ``correct'' for conditional distribution $p$.
In particular we should have $\propdis(r,p) = 0 \iff r \in \gamma(p)$.

\begin{definition}[Consistent: property]\label{def:consistent-prop}
	Suppose we are given a loss $L \in \L$, link function $\psi: \reals^d \to \R$, and property $\gamma:\P \toto \R$.
	Moreover, let $\propdis : \R \times \P \to \reals_+$ be any function satisfying $\propdis(r,p) = 0 \iff r \in \gamma(p)$.
	We say $(L, \psi)$ is \emph{$(\propdis, \D)$-consistent with respect to} $\gamma$ if, for all $D \in \D$ and sequences of measurable functions $\{f_m: \X \to \R\}$, 
	\begin{equation}
    \E_{D} L(f_m(X), Y) \to \inf_f \E_{D} L( f(X), Y) \implies \E_X \propdis(\psi \circ f_m(X), D_X) \to 0~.~
  \end{equation}
  We simply say $(L,\psi)$ is \emph{$\propdis$-consistent} if it is $(\propdis,\D)$-consistent for some $\D$ satisfying the following: for all $p \in \P$, there exists $D \in \D$ and $x \in \X$ such that $D$ has a point mass on $x$ and $p = D_x$.
  Additionally, we say $(L,\psi)$ is \emph{consistent} if there is a $\propdis$ such that $(L,\psi)$ is $\propdis$-consistent.
\end{definition}

Typical definitions of consistency require $\D$ to be the set of all distributions over $\X \times \Y$, while our conditions are much weaker.
As the main focus of this paper is lower bounds on the prediction dimension, i.e., showing that surrogates of a certain prediction dimension cannot exist, these weaker conditions translate to stronger impossibility statements.

Given a target loss $\ell$, we can define a statistic $\gamma$, the property it elicits.
Intuitively, consistency of a surrogate $L$ with respect to $\ell$ and $\gamma$ are equivalent, i.e. in both cases estimates should converge to values that minimize $\ell$-loss.
We formalize this by letting $\propdis$ be the $\ell$-regret, yielding Lemma~\ref{lem:consistent-loss-implies-prop}, proven in Appendix~\ref{app:misc-omitted-proofs}.

\begin{restatable}[]{lemma}{consistentlossimpliesprop}\label{lem:consistent-loss-implies-prop}
	Let a convex $\P \subseteq \simplex$ be given.
Given a surrogate loss $L \in \L$, link $\psi$, and target loss $\ell$, set
$\mu(r,p) := R_\ell(r,p)$.
Then there is a $\D$ such that
$(L, \psi)$ is $\D$-consistent with respect to $\ell$ if and only if $(L,\psi)$ is $(\propdis, \D)$-consistent with respect to $\gamma := \prop{\ell}$. 
\end{restatable}

Because each target loss in $\L$ elicits some property, but not all target properties can be elicited by a loss (e.g. the variance), consistency with respect to a property is the strictly broader notion.
This points to indirect elicitation as a natural necessary condition for consistency, as formalized in Theorem \ref{thm:consistent-implies-indir-elic}.

\begin{theorem}\label{thm:consistent-implies-indir-elic}
  For a surrogate $L \in \L$, if the pair $(L, \psi)$ is consistent with respect to a property $\gamma: \P \toto \R$ or a loss $\ell$ eliciting $\gamma$, then $(L, \psi)$ indirectly elicits $\gamma$.
\end{theorem}

\begin{proof}
  By Lemma~\ref{lem:consistent-loss-implies-prop}, it suffices to show the result for consistency with respect to a property $\gamma$, setting $\gamma := \prop{\ell}$ if $\ell$ is given instead.
  We show the contrapositive; suppose $(L, \psi)$ does not indirectly elicit $\gamma$, meaning we have some $p \in \P$ so that $u \in \Gamma(p)$ but $\psi(u) \not \in \gamma(p)$, where $\Gamma := \prop{L}$.
  Observe that we use the fact $\Gamma(p) \neq \emptyset$.
  By definition, if we had consistency, there must be some distribution $D$ on $\X\times\Y$ with a point mass on some $x\in\X$ and $D_x = p$.
  Consider a constant sequence $\{f_m\}$ with $f_m = f'$ such that $f'(x) = u$,
  so that $\E_D L(f_m(X), Y) = \E_{D_x} L(f_m(x),Y) = \E_p L(u,Y)$.
  Since $u \in \Gamma(p)$, we have $\E_p L(u,Y) = \inf_f \E_{D_x} L(f(x),Y) = \inf_f \E_D L(f(X),Y)$.
  In particular, we have $\E_D L(f_m(X), Y) \to \inf_f \E_D L(f(X),Y)$.
  However, we have $\E_X \propdis(\psi \circ f_m(X), D_X) = \propdis(f_m(x), p) = \propdis(\psi(u), p) \neq 0$, since $\psi(u) \not \in \gamma(p)$.
  Therefore $(L, \psi)$ is not consistent with respect to $\gamma$ (Definition~\ref{def:consistent-prop}).
\end{proof}

This result allows us to state elicitation complexity as a lower bound for convex consistency dimension.

\begin{corollary}\label{cor:elic-lb-consis-dim}
	Given a property $\gamma : \P \toto \R$ or loss $\ell:\R \times \Y \to \reals$ eliciting $\gamma$, we have $\eliccvx(\gamma) \leq \conscvx(\gamma) = \conscvx(\ell)$.
\end{corollary}

\section{Prediction Dimension of Consistent Convex Surrogates}\label{sec:char-convex}
We now turn to the question of bounding the prediction dimension of a consistent convex surrogate.
From Theorem \ref{thm:consistent-implies-indir-elic}, given a target property $\gamma$ or loss $\ell$ with $\gamma = \prop{\ell}$, this task reduces to lower bounding the prediction dimension of a convex surrogate indirectly eliciting $\gamma$.
We now explore two tools, Corollaries~\ref{cor:Pcodim-flat-single-val-prop} and~\ref{cor:Pcodim-flat-elic-relint-prop}, for proving such convex elicitation lower bounds.
The key idea, crystallized from the proofs of \citet[Theorem 16]{ramaswamy2016convex} and \citet[Theorem~9]{agarwal2015consistent}, is to consider a particular distribution~$p$ and surrogate prediction $u \in \reals^d$ with is optimal for $p$.
Theorem~\ref{lem:convex-flats-inf-dim} will show that if $d$ is small, then the level set $\{p \in \P : u \in \argmin_{u'} \exploss{L}{u'}{p}\}$ must be large; in fact, it must roughly contain a high-dimensional \emph{flat}.
By definition of indirect elicitation, there is some level set $\gamma_r$ (where $u$ is linked to $r$) containing this flat as well.
The use of this result is to leverage the contrapositive: if $\gamma$ has a level set intricate enough to not contain any high-dimensional flats, then $\gamma$ cannot have a low-dimensional consistent surrogate.

\begin{definition}[Flat]\label{def:flat-general}
  For $d\in\mathbb N$, a \emph{$d$-flat}, or simply \emph{flat}, is a nonempty set $F = \zeros{W} := \{q \in \P : \E_q W = \vec 0\}$ for some measurable $W:\Y \to \reals^d$.
\end{definition}

We state our elicitation lower bounds in Corollaries~\ref{cor:Pcodim-flat-single-val-prop} and~\ref{cor:Pcodim-flat-elic-relint-prop}, which when combined with Theorem \ref{thm:consistent-implies-indir-elic}, yield consistency bounds.
A similar result is \citet[Theorem 9]{agarwal2015consistent}, which bounds the dimension of level sets of a single-valued $\prop{L}$.
Corollaries~\ref{cor:Pcodim-flat-single-val-prop} and~\ref{cor:Pcodim-flat-elic-relint-prop} instead bound the dimension of flats contained in the level sets, an additional power which we leverage in our examples.

\begin{lemma}
  \label{lem:convex-flats-inf-dim}
  Let $\Gamma:\P \toto \reals^d$ be (directly) elicited by $L \in \Lcvx_d$ for some $d\in\mathbb{N}$.
  Let $\Y$ be either a finite set, or $\Y = \reals$, in which case we assume each $p\in\P$ admits a Lebesgue density supported on the same set for all $p\in\P$.
  \footnote{This assumption is largely for technical convenience, to ensure that $\V_{u,p}$ does not depend on $p$.
    Any such assumption would suffice, and we suspect even that condition can be relaxed.}
  For all $u\in\range\Gamma$ and $p\in\Gamma_u$, there is some $V_{u,p}:\Y\to\reals^d$ such that $p \in \zeros{V_{u,p}} \subseteq \Gamma_u$.
\end{lemma}
\begin{proof}
  As $L$ is convex and elicits $\Gamma$, we have $u \in \Gamma(p) \iff \vec 0 \in \partial \exploss{L}{u}{p}$. 
  We proceed in two cases, depending on $|\Y|$.
  
  \emph{Finite $\Y$: }
  If $\Y$ is finite, this is additionally equivalent to $\vec 0 \in \oplus_y p_y \partial L(u,y)$, where $\oplus$ denotes the Minkowski sum~\citep[Theorem 4.1.1]{hiriart2012fundamentals}.\footnote{$\partial$ represents the subdifferential $\partial f(x) = \{z : f(x') - f(x) \geq \inprod{z}{x'-x}\; \forall x' \}$.}
  Expanding, we have $\oplus_y p_y \partial L(u,y) = \{ \sum_{y\in\Y} p_y x_y \mid x_y \in \partial L(u,y) \; \forall y\in\Y\}$, and thus $W p = \sum_y p_y x_y = \vec 0$ where $W = [x_1, \ldots, x_n] \in \reals^{d\times n}$; cf.~\cite[$\mathbf{A}^m$ in Theorem 16]{ramaswamy2016convex}.
  Let $V_{u,p} : \Y \to \reals^d, y \mapsto W_y$ be the function encoding the columns of $W$.
  Observe that $\E_p V_{u,p} = \vec 0$.				
  
  \emph{$\Y=\reals$: }
  Any $L \in \Lcvx_d$ satisfies the assumptions of~\cite{ioffe1969minimization}, so we may interchange subdifferentiation and expectation.
  Specifically, letting $\V_{u,p} = \{V:\Y\to\reals^d \mid V \text{ measurable}, V(y) \in \partial L(u,y) \text{ $p$-a.s.}\}$, we have
  $\partial \E_p L(u,Y) = \{\int V(y)dp(y) \mid V\in\V_{u,p}\}$.
  As $\vec 0 \in \partial \E_p L(u,Y)$, in particular, there is some $V_{u,p} \in \V_{u,p}$ such that $\E_p V_{u,p} = 0$.
  For any $q\in\P$, as by assumption $q$ is supported on the same set as $p$, we have $V_{u,p}(y)\in\partial L(u,y)$ $q$-a.s., so that $V_{u,p}\in\V_{u,q}$.
  Thus, $\E_q V_{u,p} = 0$ implies $0\in\partial\E_q L(u,Y)$ by the above.
  
  In both cases, we take the flat $F := \zeros{V_{u,p}}$, and have $p \in F$ by construction.
  To see $F \subseteq \Gamma_u$, from the chain of equivalences above, we have for any $q\in\P$ that $q \in \zeros{V_{u,p}} \implies \vec 0 \in \partial \E_q L(u,Y) \implies u \in \Gamma(q) \implies q \in \Gamma_u$.
\end{proof}

Knowing indirect elicitation implies the existence of such a flat, we now apply  Theorem~\ref{thm:consistent-implies-indir-elic} and Lemma~\ref{lem:convex-flats-inf-dim} to construct lower bounds on convex consistency dimension.

\begin{corollary}\label{cor:Pcodim-flat-single-val-prop}
  Let target property $\gamma:\P \toto \R$ and $d\in\mathbb N$ be given.
  Let $\Y$ be either a finite set, or $\Y = \reals$, in which case we assume each $p\in\P$ admits a Lebesgue density supported on the same set for all $p\in\P$.
  Let $p \in \P$ with $|\gamma(p)| = 1$, and take $\gamma(p) = \{r\}$.
If there is no $d$-flat $F$ with $p \in F \subseteq \gamma_r$, then $\conscvx(\gamma) \geq \eliccvx(\gamma) \geq d + 1$.
\end{corollary}
\begin{proof}
	Let $(L, \psi)$ indirectly elicit $\gamma$, where $L\in\Lcvx_d$, and let $\Gamma = \prop{L}$.
	As $\Gamma$ is non-empty, there is some $u \in \Gamma(p)$.
	Since $\gamma$ is single-valued at $p$, we have $r = \psi(u)$; by Lemma~\ref{lem:convex-flats-inf-dim}, we know there is a $d$-flat $F = \zeros{V_{u,p}}$ so that $p \in F \subseteq \Gamma_u$.
	By definition of indirect elicitation, we additionally have $\Gamma_u \subseteq \gamma_r$.
	Thus, we have $p \in F \subseteq \gamma_r$.
    If no flat $F$ satisfies the above conditions, then no $L\in\Lcvx_d$ indirectly elicits $\gamma$, so $\eliccvx(\gamma) \geq d+1$, and recall $\conscvx(\gamma) \geq \eliccvx(\gamma)$ by Corollary~\ref{cor:elic-lb-consis-dim}.
\end{proof}

\begin{corollary}\label{cor:Pcodim-flat-elic-relint-prop} 
  Let an elicitable target property $\gamma:\P \toto \R$ be given, where $\P\subseteq\simplex$ is defined over a finite set of outcomes $\Y$, and let $d\in\mathbb N$.
Let $p \in \relint{\P}$.
If there is no $d$-flat $F$ with $p \in F \subseteq \gamma_r$, then $\conscvx(\gamma) \geq \eliccvx(\gamma) \geq d + 1$.
\end{corollary}
\begin{proof}
	Let $(L, \psi)$ indirectly elicit $\gamma$ and the convex function $L$ and elicit $\Gamma$.
	As $\Gamma$ is non-empty, there is some $u \in \Gamma(p)$, and suppose $r' = \psi(u)$.
	Take $F \subseteq \Gamma_u$ to be the flat that exists by Lemma~\ref{lem:convex-flats-inf-dim}.
	If $r = r'$, then $p \in F \subseteq \Gamma_u \subseteq \gamma_r$ by indirect elicitation.
	Otherwise, by Lemma~\ref{lem:set-valued-prop-flats}, for elicitable properties with $p \in \gamma_r \cap \gamma_{r'}$, we observe $p \in F\subseteq \gamma_r \iff p \in F \subseteq \gamma_{r'}$.
	
    As above, if no flat $F$ satisfies the above conditions, then no $L\in\Lcvx_d$ indirectly elicits $\gamma$, so $\conscvx(\gamma) \geq \eliccvx(\gamma) \geq d+1$, recalling Corollary~\ref{cor:elic-lb-consis-dim} for the first inequality.
\end{proof}

\section{Discrete-valued predictions}\label{sec:finite-calib}

The main known technique for lower bounds on surrogate dimensions is given by \citet{ramaswamy2016convex} for the Quadrant 1 (target loss and discrete predictions).
The proof heavily builds around the ``limits of sequences'' in the definition of calibration.
By restricting slightly to the broad class of minimizable losses $\Lcvx$, we show their bound follows relatively directly from Corollary~\ref{cor:Pcodim-flat-elic-relint-prop}.
(We conjecture that the minimizability restriction to $\Lcvx$ can be lifted; see \S~\ref{sec:conclusions}.)
\citet{ramaswamy2016convex} construct what they call the subspace of feasible dimensions and give bounds in terms of its dimension.
\begin{definition}[Subspace of feasible directions]\label{def:subspace-feas}
	The \emph{subspace of feasible directions} $\Sc_\C(p)$ of a convex set $\C \subseteq \reals^n$ at $p \in \C$ is $\Sc_\C(p) = \{ v \in \reals^n : \exists \epsilon_0 > 0 $ such that $p + \epsilon v \in \C \; \forall \epsilon \in (-\epsilon_0,\epsilon_0) \}$.
\end{definition}

\citet{ramaswamy2016convex} gives a lower bound on the dimensionality of all consistent convex surrogates, i.e. $\conscvx(\ell) \geq \|p\|_0 - \dim(\Sc_{\gamma_r}(p)) - 1$ for all $p$ and $r \in \gamma(p)$, particularly in the setting where one is given a discrete prediction problem and target loss over finite outcomes.
It turns out that the subspace of feasible directions is essentially a special case of a flat described by Lemma~\ref{lem:convex-flats-inf-dim}.
So, by making a slight restriction to the class of minimizable convex surrogates $\Lcvx$, we can derive this lower bound from our general technique in a way that we find shorter and simpler.

\begin{restatable}[\cite{ramaswamy2016convex} Theorem 18]{corollary}{hariresult}\label{cor:fsd-bound}
  Let $\ell:\R \times \Y \to \reals$ be a discrete loss eliciting $\gamma:\simplex \toto \R$ with $\Y$ finite.
  Then for all $p \in \simplex$ and $r \in \gamma(p)$,
  \begin{equation}
    \conscvx(\gamma) \geq \|p\|_0 - \dim(\Sc_{\gamma_r}(p)) - 1~.~
  \end{equation}
\end{restatable}
\begin{proof}[Sketch]
  If $\conscvx(\gamma) \leq d$, then there is a $L \in \Lcvx_d$ so that $L$ is consistent with respect to $\gamma$, and in turn, indirectly elicits $\gamma$.
  Lemma~\ref{lem:convex-flats-inf-dim} says that there is some $d$-flat $F = \zeros{V}$ such that $p \in F \subseteq \gamma_r$.
  In particular, if $p \in \relint{\simplex}$, we can see $\dim(F) = \dim(\Sc_{\gamma_r}(p))$.
  Since $\affhull(\simplex)$ has dimension $|\Y| - 1 = \|p\|_0 - 1$, by rank-nullity and $\rank(V) \leq d$ (more precisely, the corresponding linear map $q\mapsto \E_q V$) we have $d \geq \|p\|_0-1-\dim(\Sc_{\gamma_r}(p))$.
  
  When $p \not \in \relint{\simplex}$, we can project down to the subsimplex on the support of $p$, again of dimension $\|p\|_0 - 1$, and modify $L$ and $\ell$ accordingly.
  Now $p$ is in the relative interior of this subsimplex, so the above gives $\conscvx(\gamma) \geq \|p\|_0 - 1 - \dim(\Sc_{\gamma_r}(p))$, where now $\Sc$ is relative to $\reals^{\supp(p)}$.
  Finally, the feasible subspace dimension in the projected space is the same as in the original space because of $p$'s location on a face of $\simplex$.
\end{proof}

There are some cases where the bound provided by Corollaries~\ref{cor:Pcodim-flat-single-val-prop} and~\ref{cor:Pcodim-flat-elic-relint-prop} is strictly tighter than the bound provided by feasible subspace dimension in Corollary~\ref{cor:fsd-bound}.
For an example of how Corollary~\ref{cor:Pcodim-flat-single-val-prop} applies to a discrete property for which there is no target loss -- a non-elicitable property, i.e. Quadrant 2, which is not considered by \citet{ramaswamy2018consistent} -- we refer the reader to Appendix~\ref{app:omitted-examples}.

\paragraph{Example: High-confidence classification.}\label{subsec:examples-finite}
Given the target loss $\ell^{abs}(r,y) := \Ind{r \not \in \{y, \bot\}} + (1/2)\Ind{r = \bot}$,  we can consider the \emph{abstain property} it elicits, where one predicts the most likely outcome $y$ if $Pr[Y=y|x] \geq 1/2$ and ``abstain'' by predicting $\bot$ otherwise.
\citet{ramaswamy2016convex} present a convex surrogate for the abstain loss that takes as input a prediction whose dimension is logarithmic in the number of outcomes, yielding new upper bounds on $\conscvx(\ell^{abs})$ which are an exponential improvement over previous results, e.g.,~\cite{crammer2001algorithmic}.

To lower bound the dimension of convex surrogates, we can consider two different distributions; in the first, our bound yields a strict gap over the feasible subspace dimension bound, and in the second, the bounds are equal.
First, we choose $p = \bullet$ to be the uniform distribution (see Figure~\ref{fig:fsd-flats-abstain}).
In this case, the bound by feasible subspace dimension yields $\conscvx(\ell^{abs}) \geq 3 - 2 - 1 = 0$, as the feasible subspace dimension is $2$ since we are on the relative interior of the level set and simplex, as shown in Figure~\ref{fig:fsd-flats-abstain} (L).

However, consider any $1$-flat containing $\bullet$.
When intersected with the simplex, one can see that any line (a $1$-flat, since $\bullet \in \relint{\simplex}$) in the simplex through $\bullet$ also leaves the cell $\gamma_\bot$, which contains $p$.
See Figure~\ref{fig:fsd-flats-abstain} (R) for intuition; a $1$-flat through $p \in \relint{\simplex}$ would be a line in such a figure.
Therefore, we have no $1$-flat containing $p$ staying in $\gamma_\bot$, so we obtain a better lower bound, $\conscvx(\ell^{abs}) \geq 2$.
Combining this with the upper bounds given by~\cite{ramaswamy2018consistent}, we observe the bound $\conscvx(\ell^{abs}) = 2$ is tight in this case with $|\Y|=3$.

Our bounds sometimes match those of~\citep{ramaswamy2016convex}; consider the distribution ${\color{blue}\star} = (1/4, 1/4, 1/2)$, shown in Figure~\ref{fig:fsd-flats-abstain}.
The feasible subspace dimension of both $\gamma_\bot$ and $\gamma_3$ at ${\color{blue}\star}$ is $1$, since one only moves toward the distributions $(0,1/2, 1/2)$ and $(1/2, 0, 1/2)$ without leaving the level sets, and the three points are collinear in $\affhull(\simplex)$, suggesting $\Sc_{\gamma_\bot}(q) = 1$.  
This yields $\conscvx(\ell^{abs}) \geq 3 - 1- 1 = 1$.
The same line segment defines a flat contained in both $\gamma_\bot$ and $\gamma_3$, so we have $\conscvx(\ell^{abs}) \geq 1$ by Corollary~\ref{cor:Pcodim-flat-elic-relint-prop}, matching the feasible subspace dimension bound.

Bounds using $d$-flats appear to work well at distributions where previous bounds via feasible subspace dimension would have been vacuous.
In essence, flats allow us a ``global'' view of the property we are eliciting, while the feasible subspace method only permits a ``local'' look at the property, so we find our method works better for distributions in $\relint\simplex$.

\begin{figure}[htb]
	\centering
	\begin{tabular}{cc}
		\includegraphics[width=.35\textwidth]{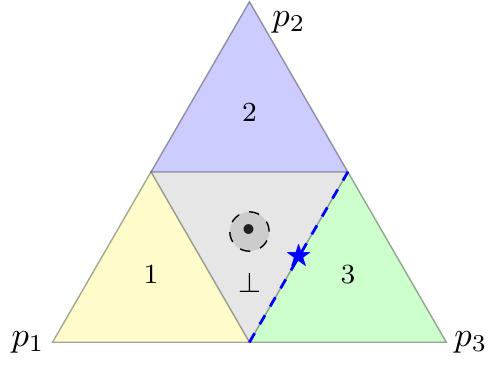} &
		\includegraphics[width=.35\textwidth]{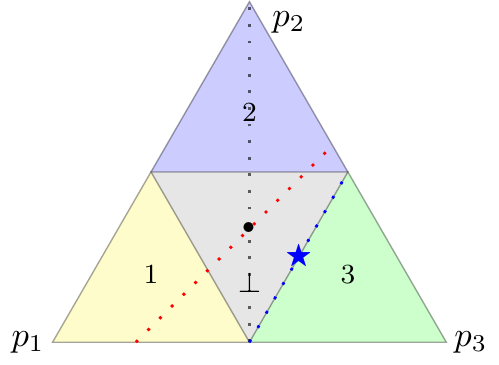} \\
	\end{tabular}
	\caption{(Left) Feasible subspace dimension $\Sc_{\gamma_\bot}(\bullet) = 2$ and $\Sc_{\gamma_\bot}({\color{blue}\star}) = 1$, giving the bound $\conscvx(\ell^{abs}) \geq 3- 1-1 = 1$.
          (Right) No $1$-flat through $\bullet$ (a line since $\bullet \in \relint{\simplex})$  stays fully contained in $\gamma_\bot$, so $\conscvx(\ell^{abs}) \geq 2$.}
	\label{fig:fsd-flats-abstain}
\end{figure}

\section{Continuous-valued predictions}\label{sec:contin-consis}

In continuous estimation problems, often one is not given a target loss, but instead a target (conditional) statistic of the data one wishes to estimate.
Examples include estimating the mean or variance of $y$ conditioned on a given $x$.
In this setting, Lemma~\ref{lem:convex-flats-inf-dim} gives lower bounds on the prediction dimension of convex losses with a link to the desired conditional statistic, i.e., the convex elicitation complexity.
In particular,
Theorem~\ref{thm:bayes-risk-lower-bound} below yields new bounds on the convex elicitation complexity of statistics which quantify risk or uncertainty such as variance, entropy, or financial risk measures.

These bounds address an open question of~\citet{frongillo2020elicitation}, that of developing a theory of elicitation complexity with respect to convex-elicitable properties.
The lower bounds of previous work are essentially all with respect to \emph{identifiable} properties;
a property is $d$-indentifiable if its level sets are all $d$-flats.
\citet{frongillo2020elicitation} rely on finding a dimension $d$ such that the level sets of certain risk measures $\gamma$ have too much curvature to contain any $d$-flat.
Thus, the elicitation complexity with respect to identifiable properties is greater than $d$.

In contrast, properties elicited by non-smooth convex losses are generally not identifiable.
For example, the properties elicited by hinge loss and the abstain surrogate are not identifiable, as their level sets are not flats (see Figure~\ref{fig:fsd-flats-abstain}).
It therefore might appear that entirely new ideas are needed.
Our framework is closely related to identifiability, however; Lemma~\ref{lem:convex-flats-inf-dim} states that the level sets of $d$-dimensional convex-elicitable properties, if not $d$-flats themselves, are unions of $d$-flats.
Thus, the general logic of~\citet{frongillo2020elicitation} can still apply.
In particular, we recover their main lower bound for the large class of Bayes risks.
\btw{Old version for finite $\Y$ is in the NeurIPS 2020 submission; especially lots more detail on variance / norm / entropy examples}

\begin{definition}
  \label{def:bayes-risk}
  Given loss function $L:\R\times\Y \to \reals$ for some report set $\R$, the \emph{Bayes risk} of $L$ is defined as $\lbar(p) := \inf_{r \in \R} \E_pL(r,Y)$.
\end{definition}

\begin{condition}\label{cond:v-interior}
  For some $r\in\range\Gamma$, the level set $\Gamma_r = \zeros{V}$ is a $d$-flat presented by some $V:\Y\to\reals^d$ such that \btw{I think we only need $\Gamma_r \subseteq \zeros{V}$.  I think I'll update the proof at some point to reflect this, since it actually gets a bit easier to follow.} $0\in\interior\{\E_pV : p\in\P\}$.
\end{condition}

\begin{restatable}{theorem}{bayesrisklowerbound}\label{thm:bayes-risk-lower-bound}
  Let $\P$ be a set of Lebesgue densities supported on the same set for all $p \in \P$.
  Let $\Gamma:\P\to\reals^d$ satisfy Condition~\ref{cond:v-interior} for some $r\in\reals^d$.
  Let $L \in \Lcvx$ elicit $\Gamma$ such that $\lbar$ is non-constant on $\Gamma_r$.
  Then $\conscvx(\lbar) \geq \eliccvx(\lbar) \geq d+1$.
\end{restatable}

We now illustrate the theorem with two important examples: variance and conditional value at risk.
Several other applications from \citet{frongillo2020elicitation}, such as spectral risk measures, entropy, and norms, follow similarly.

\paragraph{Example: Variance.}
As a warm-up, let us see how to show $\eliccvx(\Var)=2$, meaning the lowest dimension of a convex loss to estimate conditional variance is 2. 
This lower bound will follow from Theorem~\ref{thm:bayes-risk-lower-bound} using that variance is the Bayes risk of squared loss $L(r,y) = (r-y)^2$, which elicits the mean $\Gamma(p) = \E_p Y$.
Interestingly, while perhaps intuitively obvious, even this simple result is novel.
In particular, the well-known fact that the variance is not elicitable does not yield a lower bound of 2, as it does not rule out the variance being a link of a real-valued convex-elicitable property; cf.~\citet[Remark 1]{frongillo2020elicitation}.

\begin{corollary}
  \label{cor:variance}
  Let $\P$ be a set of continuous Lebesgue densities on $\Y=\reals$ with all $p \in \P$ having the same support.
  If there exist $p,q,q'\in\P$ with $\E_p Y = \E_q Y \neq \E_{q'} Y$ and $\Var(p) \neq \Var(q)$,
  \btw{Raf: I think this condition is tight actually: if there is no such triple in $\P$, I think $\eliccvx(\Var) \leq 1$ (i.e., $\Var$ is constant or a function of the mean)}
  then $\conscvx(\Var)=\eliccvx(\Var)=2$.
\end{corollary}
\begin{proof}
  For the upper bound, we may elicit the first two moments via the convex loss $L(r,y) = (r_1-y)^2 + (r_2-y^2)^2$, and recover the variance via $\psi(r) = r_2-r_1^2$, giving $\eliccvx(\Var) \leq 2$.
  Now for the lower bound.
  Without loss of generality, $\E_qY < \E_{q'}Y$.
  Let $r = \tfrac 1 2 \E_qY + \tfrac 1 2 \E_{q'}Y$, and define $V:\Y\to\reals, y\mapsto y-r$.
  Then $\zeros{V} = \{p'\in\P \mid \E_{p'}Y=r\} = \Gamma_r$ where $\Gamma:p'\mapsto \E_{p'}Y$ is the mean.
  As $\E_q Y < r < \E_{q'} Y$, we conclude $\E_q V < 0 < \E_{q'} V$.
  We have now satisfied Condition~\ref{cond:v-interior} for $d=1$.
  To apply Theorem~\ref{thm:bayes-risk-lower-bound}, it remains to show that $\Var$ is non-constant on $\Gamma_r$.
  By our assumptions and the definition of $\Var$, we have $\E_p Y^2 \neq \E_q Y^2$.
  Letting $p_1 = \tfrac 1 2 q + \tfrac 1 2 q'$, $p_2 = \tfrac 1 2 p + \tfrac 1 2 q'$, we have $\E_{p_i}Y = r$ for $i\in\{1,2\}$, but $\E_{p_1} Y^2 = \tfrac 1 2 \E_qY^2 + \tfrac 1 2 \E_{q'}Y^2 \neq \tfrac 1 2 \E_pY^2 + \tfrac 1 2 \E_{q'}Y^2 = \E_{p_2}Y^2$.
  As $p_1,p_2$ have the same mean but different second moments, we conclude $\Var(p_1) \neq \Var(p_2)$.
\end{proof}

\paragraph{Example: Conditional Value at Risk.}

\citet{frongillo2020elicitation} observe that one of the most prominent financial risk measures, the conditional value at risk (CVaR), can be expressed as a Bayes risk.
In particular, for $0 < \alpha < 1$, we may define
\begin{align}
  \CVaR_\alpha(p)
  &= \inf_{r\in\reals} \E_p\left\{ \tfrac 1 \alpha(r-Y)\ones_{r\geq Y}- r\right\}~,
  \label{eq:elic-complex-es-3}
\end{align}
which is the Bayes risk of the transformed pinball loss $L_\alpha(r,y) = \tfrac 1 \alpha(r-y)\ones_{r\geq y}-r$.
In turn, $L_\alpha$ elicits the $\alpha$-quantile, the quantity $q_\alpha(p)$ such that $\Pr_p[Y \geq q_\alpha(p)] = \alpha$.
Following \citet{frongillo2020elicitation}, we will restrict to the set $\P_q$ of probability measures over $\reals$ with connected support and whose CDFs are strictly increasing on their support, so that $q_\alpha$ is single-valued.
Under mild assumptions, we find that there is no consistent real-valued convex surrogate for $\CVaR_\alpha$.

\begin{corollary}
  \label{cor:spectral-risks}
  Let $\P$ be a set of continuous Lebesgue densities on $\Y=\reals$ with all $p \in \P$ having support on the same interval.
  If we have $p_1,p_2,p_3,p_2'\in\P$ with $q_\alpha(p_1) < q_\alpha(p_2) < q_\alpha(p_3)$ and $\CVaR_\alpha(p_2) \neq \CVaR_\alpha(p_2')$,
  then $\conscvx(\CVaR_\alpha) \geq \eliccvx(\CVaR_\alpha) \geq 2$.
\end{corollary}
As first shown by \citet{fissler2016higher}, the pair $(\CVaR_\alpha,q_\alpha)$ is jointly indentifiable and elicitable, but not by any convex loss~\citep[Prop.\ 4.2.31]{fissler2017higher}.
We conjecture the stronger statement $\eliccvx(\CVaR_\alpha) \geq 3$, which if true would constitute an interesting gap between elicitation complexity for identifiable and convex-elicitable properties.

\section{Conclusions and future work}\label{sec:conclusions}
In this work, we show that indirect property elicitation can be a powerful necessary condition for the existence of a consistent surrogate loss (Theorem~\ref{thm:consistent-implies-indir-elic}).
Furthermore, we introduce a new lower bound (Corollaries~\ref{cor:Pcodim-flat-single-val-prop} and~\ref{cor:Pcodim-flat-elic-relint-prop}) on convex consistency dimension that is generally applicable and extends previous results from both the discrete (Corollary~\ref{cor:fsd-bound}) and continuous (Corollaries~\ref{cor:variance} and \ref{cor:spectral-risks}) estimation settings.

Several important questions remain open.
Particularly for the discrete settings, we would like to know whether one can lift the restriction that surrogates always achieve a minimum; we conjecture positively.
Of course, we would like to characterize $\conscvx$ and $\eliccvx$ and develop a general framework for constructing surrogates achieving the best possible prediction dimension.
Moreover, the practical reason why consistency is desired is to ensure the guarantee of empirical risk minimization (ERM) rates; however, the relationship between ERM rates and property elicitation has not been studied.
\btw{for neurips and colt, we had commented out everything below this for space}

\newpage

\bibliography{diss,extra}

\newpage
\appendix
\section{Notes on calibration}\label{app:calibration}

When given a discrete target loss, such as for classification-like problems, direct empirical risk minimization is typically NP-hard, forcing one to find a more tractable surrogate.
To ensure consistency, the literature has embraced the notion of \emph{calibration} from~\citet[Chapter 3]{steinwart2008support}, which aligns with the definition in~\citet{tewari2007consistency} for multiclass classification, and its generalizations to arbitrary discrete target losses~\citep{agarwal2015consistent,ramaswamy2016convex}.
Calibration is more tractable and weaker than consistency, yet the two are equivalent under suitable assumptions~\citep{tewari2007consistency,ramaswamy2016convex},notably in Quadrant 1.
Intuitively, calibration says one cannot achieve the optimal surrogate loss while linking to a suboptimal target prediction.

\begin{definition}[Calibrated: Quadrant 1]\label{def:calibrated-finite}
	Let $\ell : \R \times \Y \to \reals$ be a discrete target loss.
	A surrogate loss $L : \reals^d \times \Y \to \reals$  and link $\psi:\reals^d \to \R$ pair $(L, \psi)$ is \emph{$\P$-calibrated with respect to} $\ell$ if 
	\begin{equation}\label{eq:calibration}
	\forall p \in \P: \inf_{u \in \reals^d : \psi(u) \not \in \argmin_r \E_p\ell(r,Y)} \exploss{L}{u}{p} > \inf_{u \in \reals^d} \exploss{L}{u}{p}~.~
	\end{equation}
	We simply say $L$ is calibrated if $\P = \simplex$.
\end{definition}

Many works characterize calibrated surrogates for specific discrete target losses~\citep{zhang2004statistical,lin2004note,bartlett2006convexity,tewari2007consistency}, including the canonical 0-1 loss for binary and multiclass classification.
We give another definition of calibration which is a special case of calibration via~\citet{steinwart2008support}, and show it is equivalent to Definition~\ref{def:calibrated-finite} in discrete prediction settings, but can be applied in continuous estimation settings as well.
We use this more general definition of calibration when proving statements about the relationship between consistency, calibration, and indirect elicitation.

The close connection between indirect elicitation and consistency was first explored by \citet{agarwal2015consistent}.
In particular, calibration of $L \in \L$ with respect to $\ell$ implies indirect elicitation quite directly: take $u\in\reals^d$ and $p\in\Gamma_u$, implying $u\in\Gamma(p)$.
From eq.~\eqref{eq:elic}, $\exploss{L}{u}{p} = \inf_{u'\in\reals^d} \exploss{L}{u'}{p}$, so we must have $\psi(u) \in \gamma(p)$ from eq.~\eqref{eq:calibration}, as desired.

\begin{definition}[Calibrated: Quadrants 1 and 3]\label{def:calibrated-general}
	A loss $L:\reals^d \times \Y \to \reals$ is \emph{$\P$-calibrated} with respect to a loss $\ell : \R \times \Y \to \reals$ if there is a link $\psi : \reals^d \to \R$ such that, for all distributions $p \in \P$, there exists a function $\zeta : \reals_+ \to \reals_+$ with $\zeta$ continuous at $0^+$ and $\zeta(0) = 0$ such that for all $u \in \reals^d$, we have
	\begin{equation}\label{eq:calibrated-general}
	\ell( \psi(u); p) - \risk{\ell}(p)  \leq \zeta \left(  \exploss{L}{u}{p} - \risk{L}(p) \right)~.~
	\end{equation}
	If $\P = \simplex$, we simply say $(L, \psi)$ is calibrated.
\end{definition}

Consider the following four conditions: Suppose we are given $\zeta:\reals_+ \to \reals_+$.
\begin{enumerate}
	\item [A] $\zeta$ satisfies $\zeta : 0 \mapsto 0$ and is continuous at $0$.
	\item [B] $\epsilon_m \to 0 \implies \zeta(\epsilon_m) \to 0$.
	\item [C] Given $\zeta:\reals \to \reals_+$, for all $u \in \reals^d$, $R_\ell(\psi(u); p) \leq \zeta(R_L(u;p))$.
	\item [D] For all $p \in \P$ and sequences $\{u_m\}$ so that $R_L(u_m; p) \to 0$, we have $R_\ell(\psi(u_m); p) \to 0$.
\end{enumerate}
The existence of a function $\zeta$ so that $(A \wedge C)$ defines calibration as in Definition~\ref{def:calibrated-general}, and we show $A \iff B$ in Lemma~\ref{lem:continuous-iff-limits}.  
Lemma~\ref{lem:calib-converging-regrets} shows calibration if and only if $D$, which yields a condition equivalent to calibration without dependence the function $\zeta$.

\begin{proposition}
	When $\R$ and $\Y$ are finite, a continuous loss and link $(L, \psi)$ are $\P$-calibrated with respect to a target loss $\ell$ via Definition~\ref{def:calibrated-general} if and only if they are $\P$-calibrated via Definition~\ref{def:calibrated-finite}.
\end{proposition}
\begin{proof}
$\implies$
	We prove the contrapositive; if $(L, \psi)$ is not calibrated with respect to $\ell$ by Definition~\ref{def:calibrated-finite}, then it is not calibrated via Definition~\ref{def:calibrated-general} either.
	If $(L, \psi)$ are not calibrated with respect to $\ell$ by Definition~\ref{def:calibrated-finite}, then there is a $p \in \P$ so that $\inf_{u : \psi(u) \not \in \gamma(p)} \exploss{L}{u}{p} = \inf_u \exploss{L}{u}{p}$.
	Thus there is a sequence $\{u_m\}$ so that $\lim_{m \to \infty} \psi(u_m) \not \in \gamma(p)$ and $\exploss{L}{u_m}{p} \to \risk{L}(p)$.  
	Now we have $R_L(u_m; p) \to 0$ but $R_\ell(\psi(u_m); p) \not \to 0$, so by Lemma~\ref{lem:calib-converging-regrets}, we contradict calibration by Definition~\ref{def:calibrated-general}.

$\impliedby$
Suppose there was a function $\zeta$ satisfying the bound in eq.~\eqref{eq:calibrated-general} for a fixed distribution $p \in \P$.
Observe the bound in eq.~\eqref{eq:calibration} can be written as $R_L(u,p) > 0$ for all $p \in \simplex$ and $u$ such that $\psi(u) \neq \gamma(p)$. 
By eq.~\eqref{eq:calibrated-general}, for any sequence $\{u_m\}$ so that $\psi(u_m) \not \to \gamma(p)$, we have must have $\zeta(R_\ell(\psi(u_m), p)) \not \to 0$ as we would otherwise contradict the bound in eq.~\eqref{eq:calibrated-general} since $R_\ell(\psi(u), p) \not \to 0$. 
Therefore $R_L(u_m, p) \not \to 0$; thus, the strict inequality holds.
\end{proof}

The following Lemma shows that conditions $A$ and $B$ are equivalent, so that we can using condition $B$ in lieu of condition $A$ in the proof of Lemma~\ref{lem:calib-converging-regrets}
\begin{lemma}\label{lem:continuous-iff-limits}
	A function $\zeta:\reals \to \reals$ is continuous at $0$ and $\zeta(0) = 0$ if and only if the sequence $\{u_m\} \to 0 \implies \zeta(u_m) \to 0$.
\end{lemma}
\begin{proof}
	$\implies$ Suppose we have a sequence $\{u_m\} \to 0$.
	By continuity, we have $\lim_{u_m \to 0}\zeta(u_m) = \zeta(0) = 0$, so $\zeta(u_m) \to 0$.
	
	$\impliedby$ Suppose $\zeta(0) \neq 0$ but $\zeta$ was continuous at $0$.
	The constant sequence $\{u_m\} = 0$ then converges to $0$, but as $\zeta$ is continuous at $0$, we must have $\lim_{m \to \infty}\zeta(u_m) = \zeta(0) \neq 0$, so $\zeta(u_m) \not \to 0$.
	
	Now suppose $\zeta(0) = 0$ but $\zeta$ was not continuous at $0$.
	There must be a sequence $\{u_m\} \to 0$ so that $\lim_{m \to \infty}\zeta(u_m) \neq \zeta(0) = 0$, so $\zeta(u_m) \not \to 0$.
\end{proof}

Lemma~\ref{lem:calib-converging-regrets} now gives a condition equivalent to calibration without requiring one to already have a function $\zeta$ in mind.
\begin{lemma}\label{lem:calib-converging-regrets}
	A continuous surrogate and link $(L,\psi)$ are $\P$-calibrated (via definition~\ref{def:calibrated-general}) with respect to $\ell$ if and only if, for all $p \in \P$ and sequences $\{u_m\}$ so that $R_L(u_m; p) \to 0$, we have $R_\ell(\psi(u_m); p) \to 0$.
\end{lemma}
\begin{proof}
	$\implies$ Take a sequence $\{u_m\}$ so that $R_L(u_m;p) \to 0$.
	Since $\zeta(0) = 0$ and $\zeta$ is continuous at $0$, we have $\zeta(R_L(u_m;p)) \to 0$.
	As the bound from Equation~\eqref{eq:calibrated-general} is satisfied for all $u \in \reals^d$ by assumption, we observe
	\begin{align*}
	\forall m, \; &0 \leq R_\ell(\psi(u_m); p) \leq \zeta(R_L(u_m;p))\\
	\implies &0 \leq \lim_{m \to \infty} R_\ell(\psi(u_m); p) \leq \lim_{m \to \infty} \zeta(R_L(u_m;p)) = 0\\
	\implies &0 = \lim_{m\to\infty} R_\ell(\psi(u_m); p) ~.~
	\end{align*}

	$\impliedby$ 
	Fix $p \in \P$, and consider $\zeta(c) := \sup_{u: R_L(u,p) \leq c} R_\ell(\psi(u); p)$.  
	We will show $R_L(u_m; p) \to 0 \implies R_\ell(\psi(u_m); p) \to 0$ gives calibration via the function $\zeta$ constructed above. 
	With $\zeta$ as constructed, we observe that the bound in equation~\eqref{eq:calibrated-general} is satisfied for all $u \in \reals^d$ and apply Lemma~\ref{lem:continuous-iff-limits} to observe that if there is a sequence $\{\epsilon_m\} \to 0$ so that $\zeta(\epsilon_m) \not \to 0$, it is because $R_L(u_m, p) \not \to 0 \not\implies R_\ell(\psi(u_m), p) \to 0$.

Now, we observe that the bound in Equation~\eqref{eq:calibrated-general} is satisfied for all $u \in \reals^d$ by construction of $\zeta$.
Let $S(v) := \{u' \in \reals^d : R_L(u';p) \leq R_L(v,p) \}$.
Showing $R_\ell(\psi(u);p) \leq \sup_{u' \in S(u)} R_\ell(\psi(u') ; p)$ for all $u \in \reals^d$ gives the condition $C$.
As $u$ is in the space over which the supremum is being taken (as $R_L(u;p) \leq R_L(u;p)$), we then have calibration by definition of the supremum.

Now suppose there exists a sequence $\{\epsilon_m\} \to 0$ so that $\zeta(\epsilon_m) \not \to 0$.
Consider $S(\epsilon) = \{u \in \reals^d : R_L(u,p) \leq \epsilon\}$.

\begin{align*}
\epsilon_1 \leq \epsilon_2 &\implies S(\epsilon_1) \subseteq S(\epsilon_2)\\
&\implies \zeta(\epsilon_1) \leq \zeta(\epsilon_2)~.~
\end{align*}
Now suppose there exists a sequence $\{u_m\}$ so that $R_L(u_m, p) \to 0$.
Then for all $\epsilon > 0$, there exists a $m' \in \mathbb{N}$ so that $R_L(u_m, p) < \epsilon$ for all $m \geq m'$.
Since this is true for all $\epsilon$, we have $S(\epsilon)$ nonempty for all $\epsilon > 0$, and therefore $\zeta(c)$ is discrete for all $c > 0$.
Now if $\zeta(\epsilon_m) \not \to 0$, it must be because $R_\ell(\psi(u_m), p) \not \to 0$ for some sequence converging to zero surrogate regret, and therefore we contradict the statement $R_L(u_m, p) \to 0 \implies R_\ell(\psi(u_m), p) \to 0$.

Moreover, we argue that such a sequence of $\{u_m\}$ with converging surrogate regret always exists by continuity and boundedness from below of the surrogate loss,
since we can take the constant sequence at the (attained) infimum.
\end{proof}

\subsection{Relating calibration, consistency, and indirect elicitation.}
Even with the more general notion of calibration that extends beyond discrete predictions, we still have consistency implying calibration.
\begin{proposition}\label{prop:consistent-implies-calibrated}
	If a loss and link $(L, \psi)$ are consistent with respect to a loss $\ell$, then they are calibrated with respect to $\ell$.
\end{proposition}
\begin{proof}
	We show the contrapositive.
	If $(L, \psi)$ are not calibrated with respect to $\ell$, then there is a sequence $\{u_m\}$ such that $R_L(u_m; p) \to 0$ but $R_\ell(\psi(u_m); p) \not \to 0$ via Lemma~\ref{lem:calib-converging-regrets}.
	Suppose $D \sim \X \times\Y$ has only one $x \in \X$ with $Pr_D(X = x) > 0$ so that $p := D_x$ and $\E_D f(X,Y) = \E_p f(x, Y)$.
	Consider any sequence of functions $\{f_m\} \to f$ with $f_m(x) = u_m$ for all $f_m$.
	Now we have $\E_D L(f_m(X), Y) \to \inf_f \E_D L(f(X), Y)$, but $\E_D \ell(\psi \circ f(X), Y) \not \to \inf_f \E_D \ell(\psi \circ f(X), Y)$, and therefore $(L, \psi)$ is not consistent with respect to $\ell$.
\end{proof}

Moreover, we have calibration implying indirect elicitation.
\begin{lemma}\label{lem:calib-implies-indir}
	If a surrogate and link $(L, \psi)$ with $L \in \L$ are calibrated with respect to a loss $\ell:\R \times\Y \to \reals$, then $L$ indirectly elicits the property $\gamma := \prop{\ell}$.
\end{lemma}
\begin{proof}
	Let $\Gamma$ be the unique property directly elicited by $L$, and fix $p \in \simplex$ with $u$ such that $p \in \Gamma_u$.
	We know such a $u$ exists since $\Gamma(p) \neq \emptyset$.
	As $p \in \Gamma_u$, then $\zeta(\exploss{L}{u}{p} - \risk{L}(p)) = \zeta(0) = 0$, we observe the bound $\ell(\psi(u); p) \leq \risk{\ell}(p)$.
	We also have $\ell(\psi(u); p) \geq \risk{\ell}(p)$ by definition of $\risk{\ell}$, so we must have $\ell(\psi(u);p) = \risk{\ell}(p) = \ell(\gamma(p); p)$, and therefore, $p \in \gamma_{\psi(u)}$.
	Thus, we have $\Gamma_u \subseteq \gamma_{\psi(u)}$, so $L$ indirectly elicits $\gamma$.
\end{proof}

Combining the two results, we can observe the result of Theorem~\ref{thm:consistent-implies-indir-elic} another way: \emph{through calibration}.

\section{Reconstructing \citet[Thm.\ 16]{ramaswamy2016convex}}

\begin{lemma}\label{lem:finite-relint-dim}
  Let the $d$-flat $F\subseteq \P$ (defined over finite $\Y$) contain some $p\in\relint{\P}$.
  Then 
  \begin{enumerate}
  	\item[(i)] $p \in \relint{F}$; 
  	\item[(ii)] $\dim(\Scr_F(p)) \geq \dim(\affhull(\P)) - d$.
  \end{enumerate}
\end{lemma}
\begin{proof}
  
  As $F$ is a $d$-flat, we have some $W:\Y \to \reals^d$ such that $F = \zeros{W}$.
  Throughout, given a point (typically a distribution) $p$ and convex set $P$, we define $P_p := P - \{p\}$.
  Define $T_W:\spn(\P_p)\to\reals^d, v\mapsto \E_v W$.
  
  (i)
  Since $p \in \relint{\P}$, for all $q \in \P$, there is some small enough $\epsilon > 0 $ so that for all $\alpha \in (-\epsilon,\epsilon)$, the point $q_\alpha := p - \alpha(q - p)$ is still in $\P$.
  In particular, for $q \in F$, we claim $q_\alpha \in F$.
  As $p,q \in F$, we have $\E_pW = \E_qW = \vec 0$.
  By linearity of expectation, we then have $\E_{q_\alpha} W = \vec 0$.
  This implies $q_\alpha \in F$, and therefore $p \in \relint{F}$.
  
  (ii)
  We first show $\spn(F_p) = \Scr_F(p)$.
  First, take $v \in \Scr_F(p)$, and take $\epsilon_0$ as in the definition.
  For $\epsilon = \epsilon_0 / 2$, we then have $p + \epsilon v \in F \implies \epsilon v \in F_p$, and therefore, $v \in \spn(F_p)$.  
  Now take $v \in \spn(F_p)$.
  Since $p \in \relint{F}$ (i), we have $\vec 0 \in \relint{F_p}$.
  Therefore there is an $\epsilon_0 > 0$ so that $\epsilon v \in F_p$ for all $\epsilon \in (-\epsilon_0, \epsilon_0)$ by convexity of $F$.
  Therefore, $v \in \Scr_F(p)$, and we observe $\Scr_F(p) = \spn(F_p)$.
  
  We now show $\Scr_F(p) = \ker(T_W)$.
  Observe that $\Scr_F(p) \subseteq \ker(T_W)$ follows trivially from the definitions of the two functions. 
  Now let $v \in \ker(T_W)$, and $v' \in F_p$.
  This means $\E_v W = \vec 0$, so it suffices to show $v = c v' \in F_p$, thus showing $v \in \Scr_F(p)$.
  Since $p \in \relint{\P}$, we must have $\vec 0 \in \relint{F_p}$, so we know there is some small enough $\epsilon > 0$ so that $-\alpha v' \in F_p$ for $\alpha \in (-\epsilon, \epsilon)$.
  Take $c = -\alpha$, and we conclude $v \in \Scr_F(p)$.
  Therefore, $\ker(T_W) = \Scr_F(p)$.

  We finally want to show $\dim(\affhull(\P)) = \dim(\spn(\P_p))$.
  Consider that any $q \in \spn(\P_p)$ can be written as a scalar multiple of an element of $\P_p$, which can be written as a convex combination of elements of the minimal basis $\P_p$.
  In particular, since $\vec 0 \in \P_p$, it can be written as an affine combination of elements of the basis, so $\dim(\affhull(\P)) \geq \dim(\spn(\P_p))$.
  We also have $\affhull(\P) - \{p\} \subseteq \spn(\P_p)$, so $\dim(\affhull(\P)) = \dim(\affhull(\P) - \{p\}) \leq \spn(\P_p)$.
  Therefore, $\dim(\affhull(\P)) = \dim(\spn(\P_p))$.

  As $\Y$ is a finite set, $\spn(\P_p)$ is a finite-dimensional vector space.
  The rank-nullity theorem states $\dim(\im(T_W)) + \dim(\ker(T_W)) = \dim(\spn(\P_p)) = \dim(\affhull(\P))$.
  As $\dim(\im(T_W)) \leq d$, and we have shown above that $\Scr_F(p) = \spn(F_p) = \ker(T_W)$, the conclusion follows.
\end{proof}

\hariresult*
\begin{proof}
  Let $L \in \Lcvx_d$ be a calibrated surrogate for $\ell$, and let $\Gamma := \prop[\simplex]{L}$.
  Consider $\Y' := \{y\in\Y : p_y > 0\}$ and $p' = (p_y)_{y\in\Y'} \in \Delta_{\Y'}$.
  Take $L' := L|_{\Y'}$ and $\ell' := \ell|_{\Y'}$.
  Define $h:\reals^{\Y'} \to \reals_\Y$ such that $h(q') = q$ such that $q_y = q'_y$ for $y\in\Y'$ and $q_y = 0$ otherwise.
  Take $\Gamma' = \Gamma \circ h$, $\gamma' = \gamma \circ h$.
  
  We wish to first  show $L'$ indirectly elicits $\gamma'$.
  Since $L$ indirectly elicits $\gamma$, we have a link $\psi$ such that for all $u \in \reals^d$, $\Gamma_u \subseteq \gamma_{\psi(u)}$.
  As $\Gamma'(q) = \Gamma(h(q))$ and $\gamma'(q) = \gamma(h(q))$, we have $q \in \Gamma'_u \iff h(q) \in \Gamma_u \implies h(q) \in \gamma_{\psi(u)} \iff (q_y)_{y \in \Y'} \in \gamma'_{\psi(u)}$, and therefore, $L'$ indirectly elicits $\gamma'$ via the link $\psi \circ \proj(\Y')$, where $\proj(\Y') : q \mapsto (q_y)_{y \in \Y'}$.

We aim to show $\dim(\Sc_{\gamma_r}(p)) \geq \dim(\Sc_{\gamma'_r}(p'))$.
  We do this by showing that $h(\Sc_{\gamma'_r}(p')) \subseteq \Sc_{\gamma_r}(p)$, and the result holds as $h$ is linear and injective.
Suppose $v \in h(\Sc_{\gamma'_r}(p'))$, then there exists a $v'$ so that $v = h(v')$ and an $\epsilon_0 > 0$ such that $\epsilon v' + p' \in \gamma'_r$ for all $\epsilon \in (-\epsilon_0, \epsilon_0)$.
Since $h$ is linear and recall $h(\gamma'_r) \subseteq \gamma_r$, this implies $\epsilon v + p \in \gamma_r$ for all $\epsilon \in (-\epsilon_0, \epsilon_0)$.
Therefore $v \in \Sc_{\gamma_r}(p)$, and the result follows.

  As $L'$ indirectly elicits $\gamma'$, by Corollary~\ref{cor:Pcodim-flat-elic-relint-prop}, we know there exists a $d$-flat $F$ with $p' \in F \subseteq \gamma'_r$.
  Taking $\P = \simplexp$, we know $p' \in \relint{\simplexp}$ by construction, so we can apply Lemma~\ref{lem:finite-relint-dim}(ii), which gives $\dim(\Scr_F(p')) \geq  \dim(\affhull(\simplexp)) - d = \|p\|_0 - 1-d$.
  \footnote{To reason about $\dim(\affhull(\simplexp)) = \|p\|_0 - 1$, observe that the uniform distribution on $\simplexp$ has full support and therefore requires $\|p\|_0 - 1$ elements in its basis.}
  Additionally, $\Scr_F(p') \subseteq \Sc_{\gamma'_r}(p')$ by subset inclusion of the sets themselves.
  Chaining these results, we obtain 
  \begin{align*}
    \dim(\Sc_{\gamma_r}(p)) &\geq \dim(\Sc_{\gamma'_r}(p'))  \geq \dim(\Scr_F(p')) \geq \|p\|_0 - 1 - d~.~
  \end{align*}
\end{proof}

\section{Proof of Theorem~\ref{thm:bayes-risk-lower-bound}}\label{app:pf-bayesrisklowerbound}

\newcommand{\EL}{\mathcal{E}}
\newcommand{\defeq}{:=}
\newcommand{\conv}{\mathrm{conv}}

\subsection{General setting of elicitation complexity}

We briefly introduce the general notion of elicitation complexity, of which Definition~\ref{def:cvx-elic-complex} is a special case, as some statements are more naturally made in this general setting.

\begin{definition}
  \label{def:refine}
  $\Gamma'$ \emph{refines} $\Gamma$ if for all $r'\in\range\Gamma'$ there exists $r\in\range\Gamma$ with $\Gamma'_{r'} \subseteq \Gamma_r$.
\end{definition}
Equivalently, $\Gamma'$ refines $\Gamma$ if there is a link function $\psi:\range\Gamma'\to\range\Gamma$ such that $\Gamma'_{r'} \subseteq \Gamma_{\psi(r')}$ for all $r'\in\range\Gamma'$.

\begin{definition}
  \label{def:el}
  For $k\in\mathbb{N}\cup\{\infty\}$, let $\EL_k(\P)$ denote the class of all elicitable properties $\Gamma:\P\to\reals^k$, and $\EL(\P) \defeq \bigcup_{k\in\mathbb{N}\cup\{\infty\}} \EL_k(\P)$.
  When $\P$ is implicit we simply write $\EL$.
\end{definition}

\begin{definition}
  \label{def:elic-complex}
  Let $\C$ be a class of properties.
  The \emph{elicitation complexity} of a property $\Gamma$ with respect to $\C$, denoted  $\elic_\C(\Gamma)$, is the minimum value of $k\in\mathbb{N}\cup\{\infty\}$ such that there exists $\hat \Gamma \in \C\cap\EL_k(\P)$ that refines $\Gamma$.
\end{definition}

\subsection{Supporting statements}

\begin{proposition}[\citet{osband1985providing}]
  \label{prop:elic-complex-level-sets-convex}
  Let $\Gamma$ be elicitable.
  Then $\Gamma_r$ is convex for all $r\in\range\Gamma$.
\end{proposition}

\begin{lemma}[{Set-valued extension of \citet[Lemma 4]{frongillo2020elicitation}}]
  \label{lem:refine}
  If $\Gamma'$ refines $\Gamma$ then $\elic_\C(\Gamma') \geq \elic_\C(\Gamma)$.
\end{lemma}

\begin{proof}
  As $\Gamma'$ refines $\Gamma$, we have some $\psi:\range\Gamma'\to\range\Gamma$ such that for all $r'\in \range\Gamma'$ we have $\Gamma'_{r'} \subseteq \Gamma_{\psi(r')}$.
  Suppose we have $\hat\Gamma\in\C$ and $\varphi:\range\hat\Gamma\to\range\Gamma'$ such that for all $u\in \range\hat\Gamma$ we have $\hat\Gamma_u \subseteq \Gamma'_{\varphi(u)}$.
  Then for all $u\in \range\hat\Gamma$ we have $\hat\Gamma_u \subseteq \Gamma'_{\varphi(u)} \subseteq \Gamma_{(\psi \circ \varphi)(u)}$.
  In particular, if $\elic_\C(\Gamma') = m$, then we have such a $\hat\Gamma:\P\toto \reals^m$, and hence $\elic_\C(\Gamma) \leq m$.
\end{proof}

\begin{lemma}[{\citet[Lemma 8]{frongillo2020elicitation}}]
  \label{lem:elic-complex-bayes-concave}
  Suppose $L \in \L$ elicits $\Gamma:\P\to\R$ and has Bayes risk $\lbar$.
  Then for any $p,p'\in\P$ with $\Gamma(p)\neq\Gamma(p')$, we have $\lbar(\lambda p + (1-\lambda) p') > \lambda \lbar(p) + (1-\lambda) \lbar(p')$ for all $\lambda\in(0,1)$.
\end{lemma}

\begin{lemma}[{Adapted from \citet[Theorem 4]{frongillo2020elicitation}}]
  \label{lem:bayes-risk-lower-bound}
  If $L$ elicits a single-valued $\Gamma$, and $\hat\Gamma$ refines $\lbar$, then $\hat\Gamma$ refines $\Gamma$.
\end{lemma}
\begin{proof}
  Suppose for a contradiction that $\hat\Gamma$ does not refine $\Gamma$.
  Then we have some $u\in\range\hat\Gamma$ such that for all $r\in\range\Gamma$ we have $\hat\Gamma_u \not\subseteq \Gamma_r$.
  In particular, recalling that $\Gamma$ is single-valued, we must have $p,p'\in\hat\Gamma_u$ such that $\Gamma(p) \neq \Gamma(p')$.
  Moreover, as $\hat\Gamma$ refines $\lbar$, we also have $\lbar(p) = \lbar(p')$.
  From Lemma~\ref{lem:elic-complex-bayes-concave} and $\lambda=1/2$ we have $\lbar(q) > \tfrac 1 2 \lbar(p) + \tfrac 1 2 \lbar(p') = \lbar(p)$, where $q = \tfrac 1 2 p + \tfrac 1 2 p'$.
  As the level set $\hat\Gamma_u$ is convex by Proposition~\ref{prop:elic-complex-level-sets-convex}, we also have $q \in \hat\Gamma_u$, and hence $\lbar(q)=\lbar(p)$, a contradiction.
\end{proof}

\begin{lemma}[Minor modifications from \citet{frongillo2020elicitation}]\label{lem:lin-alg-nested-level-sets}
  Let $\V$ be a real vector space.
  Let $f:\V\to\reals^k$ be linear and $C\subseteq \V$ convex with $\spn C = \V$, and let $m\in\mathbb{N}$.
  Suppose that $0 \in \interior f(C)$, and
  for all $v\in S \defeq C \cap \ker f$, there exists a linear $\hat f_v:\V\to\reals^m$ with $v \in C \cap \ker \hat f_v \subseteq S$.
  Then $m \geq k$.
  If $m=k$, we additionally have $0\in\interior\hat f_v(C)$ for some $v\in S$.
\end{lemma}
\begin{proof}
  The condition $0 \in \interior f(C)$ is equivalent to the existence of some $v_1,\ldots v_{k+1} \in C$ such that $0\in\interior\conv\{f(v_i) : i\in\{1,\ldots,k+1\}\}$.
  Let $\alpha_1,\ldots,\alpha_{k+1}>0$, $\sum_{i=1}^{k+1} \alpha_i = 1$, such that $\sum_{i=1}^{k+1} \alpha_i f(v_i) = 0$.
  As these are barycentric coordinates,
  this choice of $\alpha_i$ is unique, a fact which will be important later.
  We will take $v = \sum_{i=1}^{k+1} \alpha_i v_i$, an element of $C$ by convexity, and thus an element of $S$ as $f(v)=0$.

  Let $\hat f_v:\V\to\reals^m$ be linear with $v \in \hat S := C \cap \ker \hat f_v \subseteq S$.
  Let $\beta_1,\ldots,\beta_{k+1}\in\reals$, $\sum_{i=1}^{k+1} \beta_i = 0$, such that $\sum_{i=1}^{k+1} \beta_i \hat f_v(v_i) = 0$.
  We will show that the $\beta_i$ must be identically zero, i.e. that $\{\hat f_v(v_i):i\in\{1,\ldots,k+1\}\}$ are affinely independent.
  By construction, $v' := \sum_{i=1}^{k+1} \beta_iv_i \in \ker \hat f_v$, and as $v\in\ker\hat f_v$, for all $\lambda > 0$ we have $v_\lambda := v + \lambda v' = \sum_{i=1}^{k+1} (\alpha_i + \lambda \beta_i) v_i \in \ker \hat f_v$.
  Taking $\lambda$ sufficiently small, we have $\gamma_i := \alpha_i + \lambda \beta_i > 0$ for all $i$, and $\sum_{i=1}^{k+1} \gamma_i = \sum_{i=1}^{k+1} \alpha_i + \lambda\sum_{i=1}^{k+1} \beta_i = 1$.
  By convexity of $C$, we have $v_\lambda \in C$.
  Now $v_\lambda \in C \cap \ker \hat f_v \subseteq S = C \cap \ker f$, and in particular $v_\lambda \in \ker f$.
  Thus, $f(v_\lambda) = \sum_{i=1}^{k+1} \gamma_i f(v_i) = 0$.
  By the uniqueness of barycentric coordinates, for all $i\in\{1,\ldots,k+1\}$, we must have $\gamma_i = \alpha_i$ and thus $\beta_i = 0$, as desired.

  As $\hat f_v(C)$ contains $k+1$ affinely independent points, we have $m \geq \dim \im \hat f_v \geq k$.
  When $m=k$,
  by affine independence, the set $\conv\{\hat f_v(v_i) : i\in\{1,\ldots,k+1\}\}$ has dimension $k$ in $\reals^k$.
  As $0 = \hat f_v(v) = \sum_{i=1}^{k+1} \alpha_i \hat f_v(v_i)$, and $\alpha_i > 0$ for all $i$, we conclude $0\in\interior\conv\{\hat f_v(v_i) : i\in\{1,\ldots,k+1\}\} \subseteq \interior \hat f_v(C)$.
\end{proof}

\begin{lemma}[{\citet[Lemma 14]{frongillo2020elicitation}}]
  \label{lem:lin-alg-span}
  Let $\V$ be a real vector space.
  Let $f:\V\to\reals^k$ be linear, $C\subseteq \V$ convex with $\spn C = \V$, and let $S = C \cap \ker f$.
  If $0 \in \interior f(C)$ then $\spn S = \ker f$.
\end{lemma}

\subsection{Proving the lower bound for spectral risks}

Let $\C^*_d$ be the class of properties $\Gamma$ which are elicited by a convex loss $L\in\Lcvx_d$ for some $d\in\mathbb{N}$, and let $\C^*\defeq \bigcup_{d\in\mathbb{N}} \C^*_d$.
Then for all properties $\gamma$, if $\elic_{\C^*}(\gamma) < \infty$, we have $\elic_{\C^*}(\gamma) = \eliccvx(\gamma)$, a fact we use tacitly in the proof.

\bayesrisklowerbound*
\begin{proof}
  Let $V:\Y\to\reals^d$ and $r$ be given by the statement of the theorem and from Condition~\ref{cond:v-interior}.
  Let $m = \elic_{\C^*}(\lbar)$, so that we have $\hat\Gamma\in\C^*_m$ which refines $\lbar$.
  By Lemma~\ref{lem:bayes-risk-lower-bound} we have $\hat\Gamma$ refines $\Gamma$.

  We now establish the conditions of Lemma~\ref{lem:lin-alg-nested-level-sets} for $C=\P$.
  Let $f:\spn \P \to \reals^d$, $p \mapsto \E_pV$.
  From Condition~\ref{cond:v-interior}, we have $0\in\interior f(\P)$ and $\ker f \cap \P = \zeros{V} = \Gamma_r$.
  Now let $p\in\Gamma_r$ be arbitrary, and take any $u\in\hat\Gamma(p)$.
  As $\Gamma$ is single-valued, $r\in\range\Gamma$ is the unique value with $p\in\Gamma_r$.
  As $\hat\Gamma$ refines $\Gamma$, there exists $r'\in\range\Gamma$ with $\hat\Gamma_u \subseteq \Gamma_{r'}$, and since $p\in\hat\Gamma_u$, we conclude $r'=r$ from the above.
  From Lemma~\ref{lem:convex-flats-inf-dim}, we have some $\hat V_{u,p}$ with $p\in\zeros{\hat V_{u,p}} \subseteq \hat\Gamma_u \subseteq \Gamma_r = \zeros{V}$.
  Letting $\hat f_p:\spn \P \to \reals^d$, $p \mapsto \E_p\hat V_{u,p}$, we have now satisfied the conditions of Lemma~\ref{lem:lin-alg-nested-level-sets}.
  We conclude $m \geq d$, and moreover, if $m=d$, then there exists some $q\in\Gamma_r$ such that $0 \in\interior\hat f_q(\P)$.
    
  Now suppose $m = d$ for a contradiction.
  Let $\hat S\defeq \ker f_q\cap \P$.
  Applying Lemma~\ref{lem:lin-alg-span} to the functions $f$ and $\hat f_q$
  we have $\spn \ker f = \spn \Gamma_r$ and $\spn \ker \hat f_q = \spn \hat S$.
  As $\hat S \subseteq \Gamma_r$, we have $\ker \hat f_q = \spn \hat S \subseteq \spn \Gamma_r = \ker f$.
  By the first isomorphism theorem, we also have $\codim \ker \hat f_q = \codim \ker f = d$, as the images of these linear maps span all of $\reals^d$.
  By the third isomorphism theorem we conclude $\Gamma_r = \hat S$.
  Moreover, as $\hat S \subseteq \hat\Gamma_u \subseteq \Gamma_r$, we have $\hat S = \hat\Gamma_u = \Gamma_r$.

  We now see that $\lbar$ is constant on $\Gamma_r$ since there is some link function $\psi:\reals^m\to\reals$ such that $\Gamma_r = \hat\Gamma_u \subseteq \lbar_{\psi(u)}$, meaning $\lbar(p) = \psi(u)$ for all $p\in\Gamma_r$.
  This statement contradicts the assumption that $\lbar$ is non-constant on $\Gamma_r$.
\end{proof}

\section{Miscellaneous omitted proofs}\label{app:misc-omitted-proofs}

\setcounter{theorem}{6}
\consistentlossimpliesprop*
\begin{proof}
	First, observe that $\propdis(r,p) = 0 \iff \exploss{\ell}{r}{p} = \inf_{r' \in \R} \exploss{\ell}{r'}{p} \iff r \in \gamma(p)$.
	Now suppose $(L, \psi)$ are consistent with respect to $\ell$, and take any sequence $\{f_m\}$ of measurable hypotheses.
	Rewriting the right-hand side of Definition~\ref{def:consistent-ell},
	\begin{align}
	&\; \E_D \ell(\psi \circ f_m(X), Y)\to \inf\nolimits_f \E_D \ell(\psi \circ f(X), Y)   \label{eqn:cons-loss-cond} \\
	&\iff \E_X R_\ell(\psi \circ f_m(X), D_X) \to 0                               \nonumber  \\
	&\iff \E_X \propdis(\psi \circ f_m(X), D_X) \to 0~.~                          \label{eqn:cons-prop-cond}
	\end{align}
	Therefore, $\mathbb{E}_D L(f_m(X),Y) \to \inf_f \mathbb{E}_D L(f(X),Y)$ implies (\ref{eqn:cons-loss-cond}) if and only if it implies (\ref{eqn:cons-prop-cond}).
	Observe that the assumptions on $\L$ allow us to apply the Fubini-Tonelli Theorem~\cite[Theorem 2.37]{folland1999real}, which yields the equivalence of eq.~\ref{eqn:cons-loss-cond} to the next line.
\end{proof}
\setcounter{theorem}{37}

A hyperplane weakly separates two sets if its two closed halfspaces respectively contain the two sets.
\begin{lemma}\label{lem:intersect-levelsets}
	If $\gamma: \P \toto \R$ is an elicitable property, then for any pair of predictions $r, r' \in \R$ where $\gamma_r \neq \gamma_{r'}$, there is a hyperplane $H = \{x \in \reals^{\Y} : v \cdot x = 0\}$, for some $v \in \reals^\Y$, that weakly separates $\gamma_r$ and $\gamma_{r'}$ and has $\gamma_r \cap H = \gamma_{r'} \cap H = \gamma_r \cap \gamma_{r'}$.
\end{lemma}
\begin{proof}
	Let $\ell$ elicit $\gamma$.
	Let $v = \ell(r, \cdot) - \ell(r', \cdot)$, interpreted as a nonzero vector in $\reals^\Y$.
	Let $H = \{ q : v \cdot q = 0 \}$.
	If $v \cdot q < 0$, then $r'$ cannot be optimal, so $q \not\in \gamma_{r'}$.
	So $\gamma_{r'} \subseteq \{ q : v \cdot q \geq 0 \}$.
	Symmetrically, $\gamma_r \subseteq \{ q : v \cdot q \leq 0 \}$.
	This is weak separation, and it immediately implies that $\gamma_r \cap \gamma_{r'} \subseteq H$.
	Finally, if and only if $v \cdot q = 0$, i.e. $q \in H$, by definition the expected losses of both reports are the same.
	So $q \in \gamma_r \cap H \iff q \in \gamma_{r'} \cap H$.
	This gives $\gamma_r \cap H = \gamma_{r'} \cap H = \gamma_r \cap \gamma_{r'} \cap H = \gamma_r \cap \gamma_{r'}$.
\end{proof}

\begin{lemma}\label{lem:set-valued-prop-flats}
	Suppose we are given an elicitable property $\gamma : \P \toto \R$, where $\Y$ is finite, and distribution $p \in \relint\P$ such that $p \in \gamma_r \cap \gamma_{r'}$ for $r,r' \in \R$.
	Then for any flat $F$ containing $p$, $F \subseteq \gamma_r \iff F \subseteq \gamma_{r'}$.
\end{lemma}
\begin{proof}
	If $\gamma_r = \gamma_{r'}$, we are done.
	Otherwise, Lemma \ref{lem:intersect-levelsets} gives a hyperplane $H = \{ x \in \reals^\Y : v \cdot x = 0\}$ and a guarantee that $\gamma_r \subseteq \{ q \in \simplex : v \cdot q \leq 0\}$, while $\gamma_{r'} \subseteq \{ q \in \simplex : v \cdot q \geq 0 \}$, and finally $\gamma_r \cap \gamma_{r'} \subseteq H$.
	
	Suppose $F \subseteq \gamma_r$; we wish to show $F \subseteq \gamma_{r'}$.
	Let $q \in F$.
	By Lemma~\ref{lem:finite-relint-dim}(i), we have $p \in \relint{F}$, so there exists $\epsilon > 0$ so that $q' = p - \epsilon (q-p) \in F$.
	
	Now, suppose for contradiction that $q \not\in \gamma_{r'}$.
	Then $v \cdot q < 0$: containment in $\gamma_r$ gives $v \cdot q \leq 0$, and if $v \cdot q = 0$ then $q \in \gamma_r \cap H \implies q \in \gamma_{r'}$, a contradiction.
	But, noting that $p \in H$, we have $v \cdot q' = -\epsilon (v \cdot q) > 0$, so $q'$ is not in $\gamma_{r}$.
	This contradicts the assumption $F \subseteq \gamma_r$.
	Therefore, we must have $q \in \gamma_{r'}$, so we have shown $F \subseteq \gamma_{r'}$.
	Because $r$ and $r'$ were completely symmetric, this completes the proof.
\end{proof}

\section{Omitted Examples}\label{app:omitted-examples}
\paragraph{Discrete problem with no target loss (Quadrant 2).}
Consider the following scenario where someone is deciding how to dress for the weather based on a meteorologist's forecast.
Consider the three outcomes $\Y = \{$rainy, sunny, snowy$\}$, and we suppose we want to have some bias towards health and safety, so the meteorologist should only predict sunny weather if $Pr[$sunny $|$ weather data$] \geq 3/4$.
Otherwise, they should predict whatever is more likely  given the weather data: rain or snow.

We can now model this problem by a property with the reports $\R = \Y$, and have 
\begin{align*}
\gamma(p) &= \begin{cases}
\text{sunny} & p_{\text{sunny}} \geq 3/4 \\
\text{rainy} & p_{\text{sunny}} \leq 3/4 \wedge p_{\text{rainy}} \geq p_{\text{snowy}} \\
\text{snowy} & p_{\text{sunny}} \leq 3/4 \wedge p_{\text{snowy}} \geq p_{\text{rainy}} \\
\end{cases}~,~
\end{align*} 
shown in Figure~\ref{fig:t-example}.
Since the cells of elicitable properties in the simplex form a power diagram~\citep{lambert2009eliciting}, we know that there is actually \emph{no} target loss that directly elicits this problem.
Constructing a consistent surrogate for this task is ill-defined without Definition~\ref{def:consistent-prop}.
The function $\propdis(r,p) = \Ind{r \not \in \gamma(p)}$ 
now allows us to use Definition~\ref{def:consistent-prop} to think about consistent surrogates for this task.

Intuitively, since the feasible subspace dimension bound would be lowest at the distribution $p = (1/8, 3/4,1/8)$, we might want to test Corollary~\ref{cor:Pcodim-flat-single-val-prop} or Corollary~\ref{cor:Pcodim-flat-elic-relint-prop} at $p$.
However, we cannot apply either at $p$ since $\gamma(p) = \{$rainy, snowy, sunny$\}$ but the property is not elicitable.
\citet[Theorem 16]{ramaswamy2016convex} cannot draw any conclusions about this property for two reasons: first, we are given a target property instead of a target loss.
Second, since the property is not elicitable (hence why there can be no target loss), we observe $\dim(\Sc_{\gamma_{\text{rainy}}}(p)) \neq \dim(\Sc_{\gamma_{\text{sunny}}}(p))$, contradicting the requirements of~\citet[Lemma 23]{ramaswamy2016convex}.

However, our bounds from Corollary~\ref{cor:Pcodim-flat-single-val-prop} on the distribution $q = (1/8, 3/4 - \epsilon, 1/8 + \epsilon)$ for a small enough $\epsilon > 0$, which we can apply since $\gamma(q) = \{$snowy$\}$, suggest that the convex elicitation complexity $\eliccvx(\gamma) \geq 2$, since there is no way to draw a $1$-flat (a line, since $q \in \relint{\simplex}$) through $q$ while staying in just one level set on the simplex.

This example also extends to other decision-tree-like properties that do not have an explicit or easily constructed target loss.
\begin{figure}
	\centering
	\includegraphics[width=0.3\linewidth]{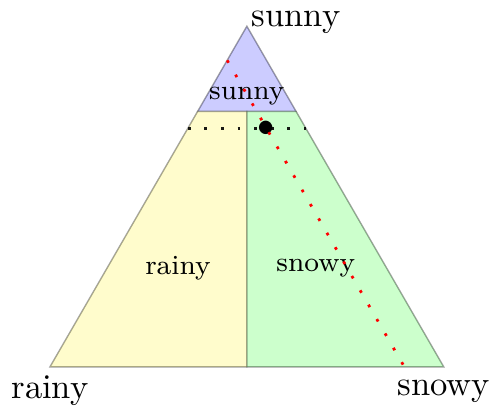}
	\caption{A meteorology example with a bias towards citizen safety.}
	\label{fig:t-example}
\end{figure}

\end{document}